\documentclass[sigconf,natbib=false]{acmart}
\AtBeginDocument{%
  }

\usepackage{multirow}
\usepackage{booktabs}
\usepackage{graphicx}
\usepackage{subcaption}

\setcopyright{none}
\renewcommand\footnotetextcopyrightpermission[1]{}
\RequirePackage[
  datamodel=acmdatamodel,
  style=acmnumeric,
  ]{biblatex}
\begin{document}
\begin{sloppy}
\title{How Is Uncertainty Propagated in Knowledge Distillation?}

\author{Ziyao Cui}
\affiliation{%
  \institution{Duke University}
  \city{Durham}
  \state{North Carolina}
  \country{USA}
}
\email{richard.cui@duke.edu}

\author{Jian Pei}
\affiliation{%
  \institution{Duke University}
  \city{Durham}
  \state{North Carolina}
  \country{USA}
}
\email{j.pei@duke.edu}







\renewcommand{\shortauthors}{Ziyao Cui and Jian Pei}

\begin{abstract}
Knowledge distillation transfers behavior from a teacher to a student model, but the process is inherently stochastic: teacher outputs, student training, and student inference can all be random. Collapsing these uncertainties to a single point estimate can distort what is learned. We systematically study how uncertainty propagates through knowledge distillation across three representative model classes—linear regression, feed-forward neural networks, and large language models (LLMs)—and propose simple corrections. We distinguish \emph{inter-student} uncertainty (variance across independently distilled students) from \emph{intra-student} uncertainty (variance of a single student’s predictive distribution), showing that standard single-response knowledge distillation suppresses intra-student variance while leaving substantial inter-student variability. To address these mismatches, we introduce two variance-aware strategies: \emph{averaging} multiple teacher responses, which reduces noise at rate $\mathcal{O}(1/k)$, and \emph{variance-weighting}, which combines teacher and student estimates via inverse-variance weighting to yield a minimum-variance estimator. We provide formal guarantees in linear regression, validate the methods in neural networks, and demonstrate empirical gains in LLM distillation, including reduced systematic noise and hallucination. These results reframe knowledge distillation as an \emph{uncertainty transformation} and show that variance-aware distillation produces more stable students that better reflect teacher uncertainty.\footnote{Our implementation and evaluation code is available at \url{https://github.com/richardcui18/uncertainty-in-distillation}.}
\end{abstract}





\maketitle

\section{Introduction}

Knowledge distillation~\cite{hinton2015distilling,kim2016sequence} is a widely used technique for transferring behavior from a high-capacity teacher model to a more efficient student model. It enjoys applications ranging from model compression~\cite{ding2023efficiency,wei2022emergent,treviso2023efficient,xu2023survey,zhu2024survey,zhuang2023survey} to domain adaptation~\cite{tang2024direct,zhang-etal-2021-matching,pan2020meta}. In practice, knowledge distillation enables models to integrate expertise across domains—for example, financial models absorbing legal reasoning, human resource systems combining financial and legal insights~\cite{shi2025knowledge,italiani2025enhancing}, or medical LLMs leveraging knowledge from pharmaceutical and bioinformatics models~\cite{hasan2025optimclm,ge2025multi}.

Much of the knowledge distillation literature emphasizes accuracy-preserving compression: distilled students often match or even exceed teacher performance on benchmarks while being substantially smaller and faster~\cite{beyer2022knowledge,fang2025knowledge,liu2024wisdom}. In this sense, knowledge distillation succeeds. Yet this focus on accuracy obscures a critical and understudied aspect of distillation—\emph{the transformation of uncertainty}. Teacher outputs are often stochastic due to sampling, temperature, or inherent ambiguity; student training depends on random initialization and optimization dynamics; and student inference may itself be probabilistic. Collapsing this distributional behavior to a single sampled response risks discarding information about uncertainty and expressiveness. In safety-critical or decision-sensitive domains such as law, medicine, and finance~\cite{zhu2024survey}, these distortions can have real consequences, including fabricated legal or medical citations, incorrect customer support responses, and inappropriate generalizations.\footnote{Mata v. Avianca, Inc. (2023): \url{https://en.wikipedia.org/wiki/Mata_v._Avianca,_Inc.}; Air Canada chatbot case (2024): \url{https://en.wikipedia.org/wiki/Hallucination_(artificial_intelligence)}; Google AI Overview (2025): \url{https://www.thetimes.com/uk/technology-uk/article/google-ai-overviews-aio-wrong-vs32029z6}}

This tension motivates a central question: does knowledge distillation truly succeed in its promise? In one sense, yes—it reliably improves efficiency and often generalization. In another sense, no—the process transfers only a fraction of the teacher’s distributional knowledge and may suppress the very uncertainties that make teachers expressive. This paradox underlies the central theme of our work: \emph{knowledge distillation works, but not always in the way its name suggests}. We therefore ask: How does knowledge distillation affect model uncertainty? Does it suppress or preserve the diversity and creativity of generative teachers? What tradeoffs arise between uncertainty reduction and fidelity to the teacher?

We address these questions by systematically analyzing how uncertainty propagates through knowledge distillation and by developing simple, principled corrections. We distinguish between two forms of uncertainty: \emph{inter-student uncertainty}, the variance across independently distilled students, and \emph{intra-student uncertainty}, the variance of a single student’s predictive distribution. Ideally, inter-student uncertainty should vanish, reflecting stable training, while intra-student uncertainty should mirror that of the teacher. Instead, we find a consistent mismatch: standard knowledge distillation suppresses intra-student uncertainty while leaving substantial inter-student variability.

Our analysis spans three representative model classes: (i) \emph{linear regression}~\cite{galton1886regression}, which permits closed-form analysis; (ii) \emph{feed-forward neural networks}~\cite{svozil1997introduction}, which test whether linear insights extend to nonlinear, non-convex settings; and (iii) \emph{large language models} (LLMs)~\cite{chang2024survey,huang2022towards,yang2024harnessing,shanahan2024talking}, where sequence-level stochasticity makes uncertainty especially salient. Across these settings, we identify common principles governing how knowledge distillation reshapes uncertainty and propose practical, scalable remedies.

Our contributions are threefold. First, we formalize uncertainty in knowledge distillation by separating inter-student and intra-student components and identifying their sources in the distillation pipeline. Second, we derive precise characterizations of how teacher noise, student initialization, and student noise influence student uncertainty in linear models and validate consistent patterns in neural networks and LLMs. Third, we introduce two variance-aware distillation strategies—averaging and variance-weighting—and prove their optimality in linear setting.

The remainder of the paper is organized as follows. Section~\ref{sec:prob-defn} formulates the problem and Section~\ref{sec:related} reviews related work. Sections~\ref{sec:uncertainty_teacher}, \ref{sec:uncertainty_init}, and~\ref{sec:uncertainty_output} analyze teacher output, student initialization, and student output uncertainty, respectively. We present uncertainty correction methods in Section~\ref{sec:solution} and conclude with discussion and future directions in Section~\ref{sec:con}. Mathematical proofs are deferred to the appendixes.

\section{Problem Definition}
\label{sec:prob-defn}

\emph{Knowledge distillation} is commonly viewed as a procedure in which a student model learns to imitate the behavior of a teacher model. This process is fundamentally stochastic: teacher model outputs may vary across samples, advanced student model training often starts from random initialization and nonconvex optimization dynamics, and generative student models may themselves produce probabilistic outputs. These sources of randomness collectively shape the distributional behavior that student models ultimately inherit. In this section, we formalize the types of uncertainty that arise in the distillation pipeline, describe how they manifest across three representative model families—linear regression, neural networks, and large language models (LLMs)—and establish the uncertainty modeling framework used throughout the remainder of this work.

\subsection{Sources and Types of Uncertainty in Knowledge Distillation}
\label{sec:types-uncertainty}

\begin{figure}[t]
\centering
\includegraphics[width=\linewidth]{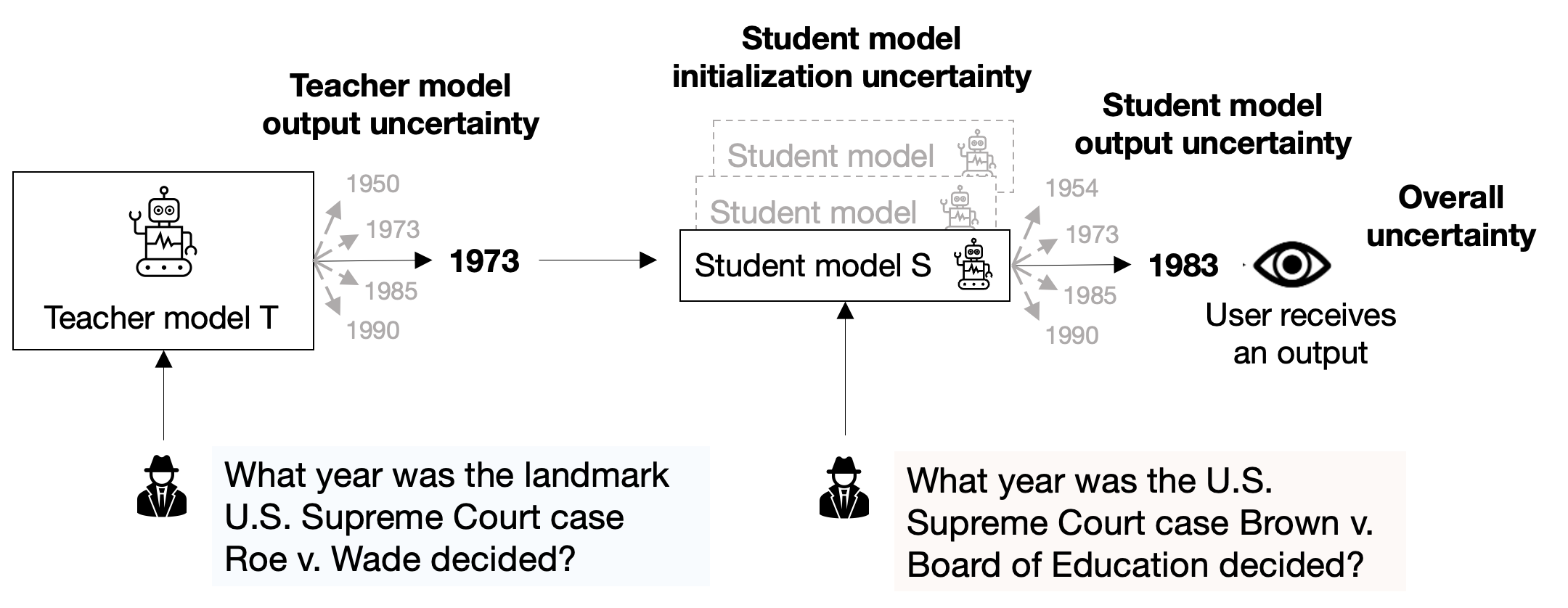}
\caption{Sources of uncertainty in knowledge distillation.}
\label{fig:conceptual}
\end{figure}

Knowledge distillation transfers predictive behavior from a teacher model to a student model. The distillation pipeline contains multiple stochastic components that introduce uncertainty into the student model’s learned behavior.  
As illustrated in Figure~\ref{fig:conceptual}, three primary sources of uncertainty arise:

\begin{itemize}
    \item \textbf{Teacher model output uncertainty:}  
    Many teacher models produce responses stochastically—due to noisy labels, sampling-based generation, temperature scaling, or inherent ambiguity in the input~\cite{evstafev2025paradox,renze-2024-effect,peeperkorn2024temperature}. Thus, each teacher model output is a draw from an underlying predictive distribution, which introduces uncertainty into the supervision signal used during knowledge distillation.
    
    \item \textbf{Student model initialization uncertainty:}  
    Training a sophisticated student model often begins with randomly initialized parameters. For nonlinear architectures such as neural networks and LLMs, different initial seeds can lead the optimization process toward different local minima. Consequently, two student models trained from identical data and supervision may still differ substantially.
    
    \item \textbf{Student model output uncertainty:}  
    Even after training, some student models—particularly generative LLMs—produce outputs via sampling from a predictive distribution. This yields stochastic output behavior, distinct from instance-level variability.
\end{itemize}

These sources fall into two broader types:  
\textbf{model output uncertainty}, the variability in predictions for a fixed model instance, and  
\textbf{model instance uncertainty}, variability across independently trained model instances.
These uncertainty sources manifest differently across model families studied in this paper.

\subsubsection{Teacher model output uncertainty}  

In linear regression and feed-forward neural networks, teacher model output uncertainty is introduced explicitly via additive noise at inference time: \(y_i^{(T)} = f_T(x_i) + \varepsilon_i\). This modeling choice ensures that teacher output uncertainty has a consistent definition across models considered.

In LLMs, teacher model output uncertainty is intrinsic: for a fixed prompt \(\mathbf{x}\), the teacher model defines a predictive distribution \(p_T(\mathbf{y} \mid \mathbf{x})\) over diverse linguistic responses. Sampling strategies such as temperature scaling further modulate this uncertainty. Standard knowledge distillation pipelines typically use only one sampled teacher model response as supervision, collapsing this distribution and discarding valuable information about uncertainty.

\subsubsection{Student model initialization uncertainty}  
Linear regression models exhibit no such uncertainty: the ordinary least-squares (OLS) objective is convex with a unique closed-form solution.

Neural networks and LLMs, in contrast, optimize highly non-convex objectives. Random initialization can cause two student models trained under identical supervision to converge to distinct parameter configurations, particularly when the distillation signal is weak. This variability directly contributes to inter-student variance.

\subsubsection{Student model output uncertainty}  
For linear regression and deterministic neural networks, once parameters are fixed, the model output is deterministic.  
LLMs, however, remain stochastic after training—sampling from student predictive distribution \(p_S(\mathbf{y} \mid \mathbf{x})\) reflects intrinsic model uncertainty and influences downstream variability. Ideally, this predictive distribution should align with that of the teacher model, but systematic mismatches often remain.

These uncertainty sources drive the systematic patterns analyzed in Sections~\ref{sec:uncertainty_teacher}, \ref{sec:uncertainty_init}, and~\ref{sec:uncertainty_output}.

\subsubsection{Problem statement}  
In this paper, we seek to understand how these different sources of uncertainty—teacher model output uncertainty, student model initialization uncertainty, and student model output uncertainty—interact and accumulate throughout the knowledge distillation pipeline. Our goal is to characterize how each source contributes to the overall uncertainty exhibited by distilled student models, and to understand when and why the student model may deviate from the teacher model’s predictive behavior. This problem formulation motivates the representative model-specific distillation setups introduced in the following subsections, which provide concrete settings in which these uncertainty mechanisms can be analyzed systematically.

\subsection{Linear Regression and Distillation}
\label{sec:prelim_models}

Linear regression models are inherently deterministic: for a fixed dataset, the OLS solution yields a unique optimal parameter vector. To use linear regression as a controlled setting for studying uncertainty in knowledge distillation, we introduce an explicit stochastic term into the teacher model.

Given an input \(x_i \in \mathbb{R}^p\), the linear regression teacher model's stochastic response is
\(
y^{(T)}_i = f_T(x_i) + \varepsilon_i\),
where $
\varepsilon_i \overset{\text{i.i.d.}}{\sim} \mathcal{N}(0, \sigma_T^2)$, \(f_T(x_i) = x_i^\top \theta_T\) is the deterministic teacher model output, and \(\sigma_T^2\) controls the injected teacher noise. The resulting distillation dataset is \(\{(x_i, y_i^{(T)})\}_{i=1}^n\).

The student model implements the linear mapping
$
f_S(x_i) = x_i^\top \theta_S$,
and is trained by minimizing the mean-squared distillation loss
$
\mathcal{L}_{\mathrm{distill}}^{\mathrm{lin}}(\theta_S)
    = \frac{1}{n}\sum_{i=1}^n \bigl(y_i^{(T)} - x_i^\top \theta_S\bigr)^{2}$.

The OLS estimator has a closed form
$
\widehat{\theta}_S = (X^\top X)^{-1} X^\top y^{(T)}$,
with variance
$
\operatorname{Var}(\widehat{\theta}_S) = \sigma_T^2 (X^\top X)^{-1}
$.
Thus, the variance of the student model parameters scales linearly with teacher model output noise.

To assess predictive accuracy, we use the evaluation mean-squared error (MSE):
\(
\mathcal{L}_{\mathrm{eval}}
   = \frac{1}{n} \sum_{i=1}^{n}
      \bigl( f_S(x_i^{\mathrm{test}}) - y_i^{\mathrm{test}} \bigr)^{2}
\).

\subsection{Neural Networks and Distillation}

In neural network settings, both teacher and student models define nonlinear mappings of the form
\(
f(x_i;\theta)
   = \phi_L\!\big(W_L\, \phi_{L-1}(\cdots \phi_1(W_1 x_i))\big)
\),
where \(\{W_\ell\}\) are layer weight matrices and \(\{\phi_\ell\}\) are activation functions.  
The teacher model outputs
\(
y_i^{(T)} = f_T(x_i) + \varepsilon_i
\),
and the student model minimizes the corresponding mean-squared distillation loss
\(
\mathcal{L}_{\mathrm{distill}}^{\mathrm{nn}}(\theta_S)
   = \frac{1}{n}\sum_{i=1}^{n}
      \bigl\| f(x_i;\theta_S) - y_i^{(T)} \bigr\|^{2}\).

Because the objective is non-convex, no closed-form solution exists. Student model parameters are optimized using full-batch gradient descent with the Adam optimizer~\cite{DBLP:journals/corr/KingmaB14}. Random initialization introduces substantial instance-level uncertainty, as independently trained student models may converge to different local minima. The predictive accuracy is assessed using the evaluation MSE $\mathcal{L}_{\mathrm{eval}}$.

\subsection{LLMs and Distillation}
\label{sec:llm prelim}

We further study uncertainty in LLMs under the free-form Question Answering (QA) setting~\cite{voorhees-tice-2000-trec,yue2025survey} within a sequence-level knowledge distillation framework~\cite{kim2016sequence}. Free-form QA involves generating linguistically rich and probabilistic responses, making uncertainty particularly salient.

An autoregressive LLM defines
\(
p(\mathbf{y} \mid \mathbf{x})
    = \prod_{t=1}^{T} P(\mathbf{y}_t \mid \mathbf{y}_{<t}, \mathbf{x}; \boldsymbol{\theta})
\),
with training performed using cross-entropy on pairs \((\mathbf{x}, \mathbf{y}^{\text{true}})\).  
During generation, sampling strategies such as temperature scaling introduce additional stochasticity.

To conduct sequence-level knowledge distillation, let \(\mathbf{x} \in \mathcal{X}\) be an input prompt and \(\mathcal{Y}\) the space of possible completions.  
The teacher model defines a sequence distribution
\(
p_T(\mathbf{y} \mid \mathbf{x})
    = \prod_{j=1}^{J} p_T(\mathbf{y}_j \mid \mathbf{y}_{<j}, \mathbf{x})
\),
and the student model parameterized by \(\boldsymbol{\theta_S}\) aims to match this distribution by solving
$
\boldsymbol{\theta_S^*}
    = \arg\min_{\boldsymbol{\theta_S}}
        \mathbb{E}_{\mathbf{x}} \Big[
            - \!\sum_{\mathbf{y}\in\mathcal{Y}}
            p_T(\mathbf{y}\mid \mathbf{x})
            \log p_S(\mathbf{y}\mid \mathbf{x};\boldsymbol{\theta_S})
        \Big]
$.

Since computing exact \(p_T(\mathbf{y} \mid \mathbf{x})\) is infeasible, standard knowledge distillation~\cite{kim2016sequence} samples a single teacher model output \(\mathbf{y}^*\) and minimizes the surrogate loss
$
-\log p_S(\mathbf{y}^* \mid \mathbf{x}; \boldsymbol{\theta_S})
$.


To quantify similarity and variability, we embed responses using SBERT~\cite{reimers-2019-sentence-bert}.  
Given embeddings \(e(\mathbf{y})\), cosine similarity is
$
\text{CosSim}(\mathbf{y}_i, \mathbf{y}_j)
    = \frac{e(\mathbf{y}_i)\cdot e(\mathbf{y}_j)}
           {\|e(\mathbf{y}_i)\|\,\|e(\mathbf{y}_j)\|}
$.

For teacher model response \(\mathbf{y}_T\) and student model responses \(\mathbf{y}_1,\ldots,\mathbf{y}_s\), we compute average alignment $\frac{1}{s}\sum_{i=1}^s \text{CosSim}(\mathbf{y}_T,\mathbf{y}_i)$ and inter-student variance $\operatorname{Var}\!\big(\text{CosSim}(\mathbf{y}_T,\mathbf{y}_i)\big)$.

\section{Related Work}
\label{sec:related}

Knowledge distillation has been widely studied as a method for transferring behavior from a high-capacity teacher model to a more efficient student model~\cite{hinton2015distilling,kim2016sequence}. Much of the prior literature focuses on model compression and transfer learning, demonstrating that distillation can preserve or even improve accuracy~\cite{beyer2022knowledge,fang2025knowledge,liu2024wisdom,tian2025beyond,khanuja2021mergedistill}. Beyond accuracy, research has examined properties inherited through distillation, including privacy~\cite{cui2025membership,jagannatha2021membership,chourasia2021knowledge,tang2022mitigating} and hallucination behavior~\cite{mcdonald2024reducing,nguyen2025smoothing,lewis2025winning}. These effects, however, can be viewed as downstream consequences of a more fundamental issue: the uncertainty present in the distillation pipeline. Our work studies this underlying factor directly, seeking to identify the sources of uncertainty and understand how they shape the behavior of the student models.

A separate line of research investigates the relationship between student model generalization and fidelity to the teacher model in the context of LLM distillation. Prior work has shown that high student model accuracy does not guarantee high agreement with the teacher model, and that fidelity is closely related to calibration~\cite{stanton2021does}. Other studies find that even when accuracy is maintained, the student model may fail to preserve structured reasoning patterns exhibited by the teacher model~\cite{ramesh2025generalization}. These analyses primarily focus on next-token prediction, which provides only a local notion of fidelity. In contrast, our study evaluates fidelity at the level of complete generated responses and connects fidelity to the propagation of uncertainty through distillation.

Recent work also examines how properties of the teacher model itself influence the outcome of distillation. Larger teacher models do not always produce better student models~\cite{cho2019efficacy}; teacher model accuracy alone is not a definitive predictor of student model quality~\cite{ramesh2025generalization}; and teacher models that are more stable or trained for longer tend to produce more reliable student models~\cite{beyer2022knowledge}. These studies underscore the importance of teacher model characteristics but address only part of the problem. Regardless of the teacher model’s strengths, every student model must pass through the distillation process, which introduces its own systematic distortions. Understanding the distillation process as a transformation of uncertainty is therefore essential: it determines not only what the student model learns correctly but also how the student model’s uncertainty may diverge from that of the teacher model.

\section{Teacher Model Output Uncertainty}
\label{sec:uncertainty_teacher}

Teacher model output uncertainty is the most direct and fundamental source of randomness in the distillation pipeline. Whenever the teacher produces noisy or probabilistic outputs, this uncertainty propagates to students through the supervision signal. Understanding this propagation is essential: it determines how much distributional variability the student inherits, how much is lost, and how strongly predictions diverge across independent distillation runs.

In this section, we study this effect analytically and empirically. We use \textbf{inter-student variance} to quantify how much independently trained students differ in their predictions when exposed to noisy teacher outputs. We begin with linear regression, where the relationship between teacher noise and student variance admits exact characterization, and then verify that the same principle extends to neural networks. Our results reveal a simple but powerful rule: \emph{inter-student variance grows linearly with the teacher model’s output noise, and this scaling persists across model families}. We conclude by examining teacher output uncertainty in LLMs.

\subsection{Analytical Results in Linear Regression}

In linear regression distillation, the student learns from noisy teacher responses of the form $y^{(T)}_i = x_i^\top \theta_T + \varepsilon_i$, where $\varepsilon_i \overset{\mathrm{i.i.d.}}{\sim} \mathcal{N}(0,\sigma_T^2)$. The student minimizes a squared loss, yielding the OLS estimator $\widehat{\theta}_S = (X^\top X)^{-1}X^\top y^{(T)}$. Because noise enters linearly into this estimator, the student’s parameters—and thus its predictions—inherit variance proportional to the teacher’s output variance.

Let $\{f_{S_1},\dots,f_{S_p}\}$ denote $p$ independently trained student models. For a test set $\{x_1,\ldots,x_n\}$, we define the \textbf{inter-student variance} as $\mathcal{V}_{\mathrm{inter}} = \frac{1}{n}\sum_{i=1}^{n}\operatorname{Var}_{j=1,\ldots,p}[f_{S_j}(x_i)]$.

\begin{theorem}[Inter-student variance scales with teacher noise]
\label{thm:inter_student_variance}
In linear regression distillation,
$\mathbb{E}\big[\mathcal{V}_{\mathrm{inter}}\big] \propto \sigma_T^2$.
\qed
\end{theorem}

Teacher model noise not only increases inter-student variance but also degrades the accuracy of individual students.

\begin{theorem}[Expected MSE relative to teacher model output grows linearly with teacher noise]
\label{thm:teacher_output_mse}
Let \(
\mathcal{L}_{\mathrm{eval}}^{(T)}
= \frac{1}{n}\sum_{i=1}^n \bigl(f_S(x_i^{\mathrm{test}}) - f_T(x_i^{\mathrm{test}})\bigr)^2
\)
denote the evaluation MSE relative to the teacher model output.
If $y^{(T)} = X\theta_T + \varepsilon$ with $\varepsilon\sim\mathcal{N}(0,\sigma_T^2 I)$, then
\[
\mathbb{E}[\mathcal{L}_{\mathrm{eval}}^{(T)}]
    = \sigma_T^2 \cdot \frac{1}{n}
      \operatorname{tr}\!\big( X_{\mathrm{test}} (X^\top X)^{-1} X_{\mathrm{test}}^\top \big),
\]
which is linear in $\sigma_T^2$.
\qed
\end{theorem}


While Theorem~\ref{thm:teacher_output_mse} characterizes the expected error of the student relative to the teacher output, practical evaluation is typically performed against ground-truth labels. The following theorem analyzes the student's expected MSE relative to the ground truth.

\begin{theorem}[Expected MSE relative to ground truth grows with teacher noise]
\label{thm:ground_truth_mse}
Assume the ground-truth targets are \(y^{\mathrm{true}} = X_{\mathrm{test}}\theta^* + \eta\), where $\eta\overset{\text{i.i.d.}}{\sim}\mathcal{N}(0,\sigma_{\eta}^2 I_n)$. Then,
\[
\mathbb{E}[\mathcal{L}_{\mathrm{eval}}]
  = \frac{1}{n}\|X_{\mathrm{test}}(\theta_T-\theta^*)\|^2
  + \sigma_T^2\,\frac{1}{n}\operatorname{tr}(X_{\mathrm{test}}(X^\top X)^{-1}X_{\mathrm{test}}^\top)
  + \sigma_\eta^2.
\]\qed
\end{theorem}

These results show that teacher model noise has a dual effect: it increases predictive disagreement across independently distilled students, and degrades the accuracy of every student in expectation. 

\subsection{Empirical Verification in Linear and Neural Models}

To test whether the analytical insights hold in linear regression and extend to nonlinear settings such as neural networks, we conduct experiments on Boston Housing dataset~\cite{harrison1978hedonic} (details in Appendix~\ref{app:datasets_all}). The teacher model is a noisy regressor with variance $\sigma_T^2 = \alpha\,\operatorname{Var}(y)$, and we vary $\alpha \in [0,2]$. For each $\alpha$, we distill 1{,}000 students and measure both mean test MSE and inter-student variance. 

Figure~\ref{fig:simple_uncertainty_1} shows that inter-student variance increases approximately linearly with $\alpha$ in the linear regression setting, consistent with Theorem~\ref{thm:inter_student_variance}. 
In neural networks, the increase is super-linear, reflecting nonlinear amplification of noise beyond the linear regime.
In both cases, student accuracy degrades monotonically as $\alpha$ grows. Together, these results confirm that while the precise scaling differs across model classes, increased teacher noise consistently leads to higher inter-student variability and reduced student performance.


\begin{figure}[t]
\centering
\begin{subfigure}[t]{0.48\linewidth}
    \centering
    \includegraphics[width=\linewidth]{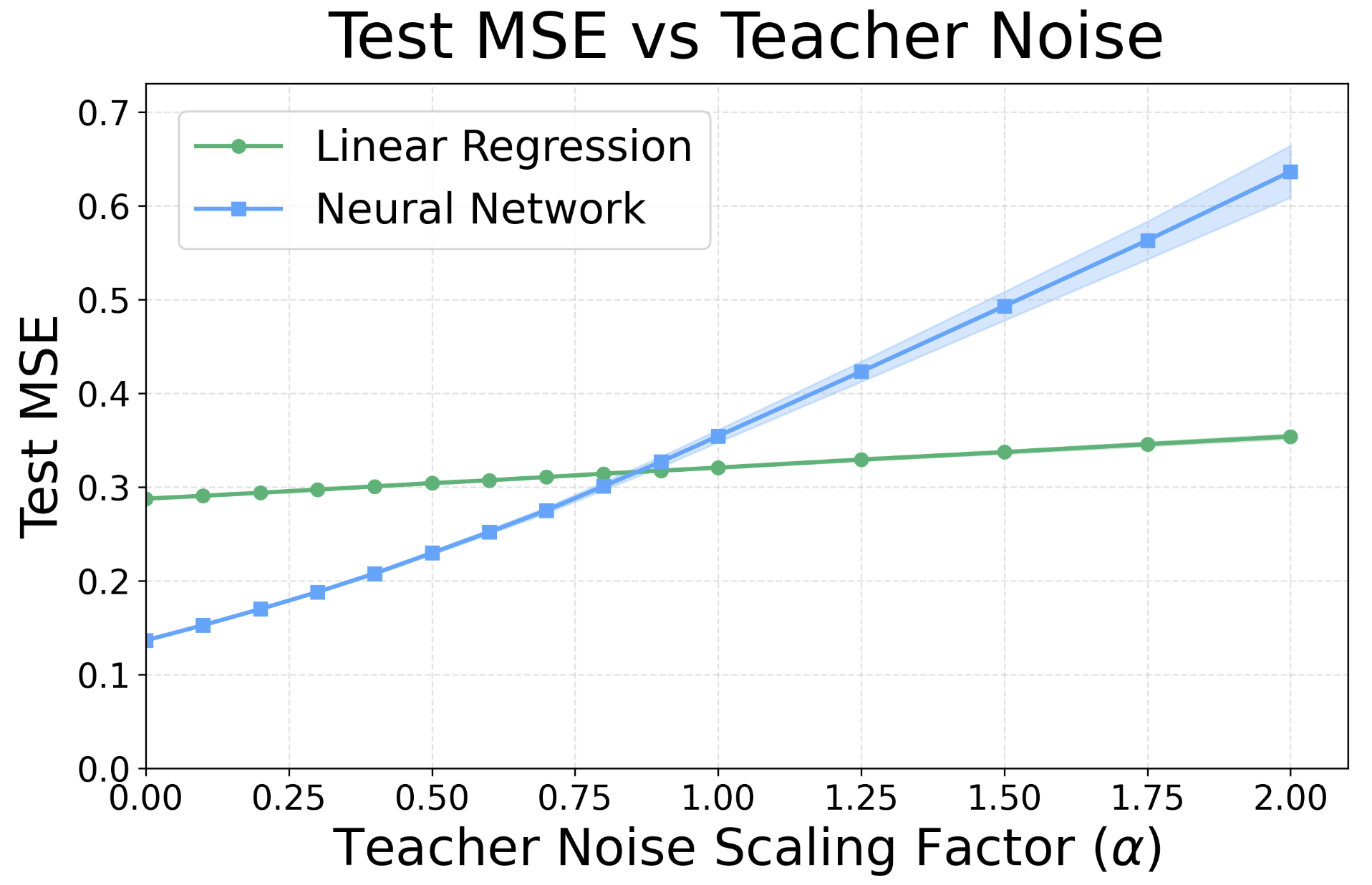}
    \caption{Test MSE.}
    \label{fig:simple_uncertainty_1_mse}
\end{subfigure}
\hfill
\begin{subfigure}[t]{0.48\linewidth}
    \centering
    \includegraphics[width=\linewidth]{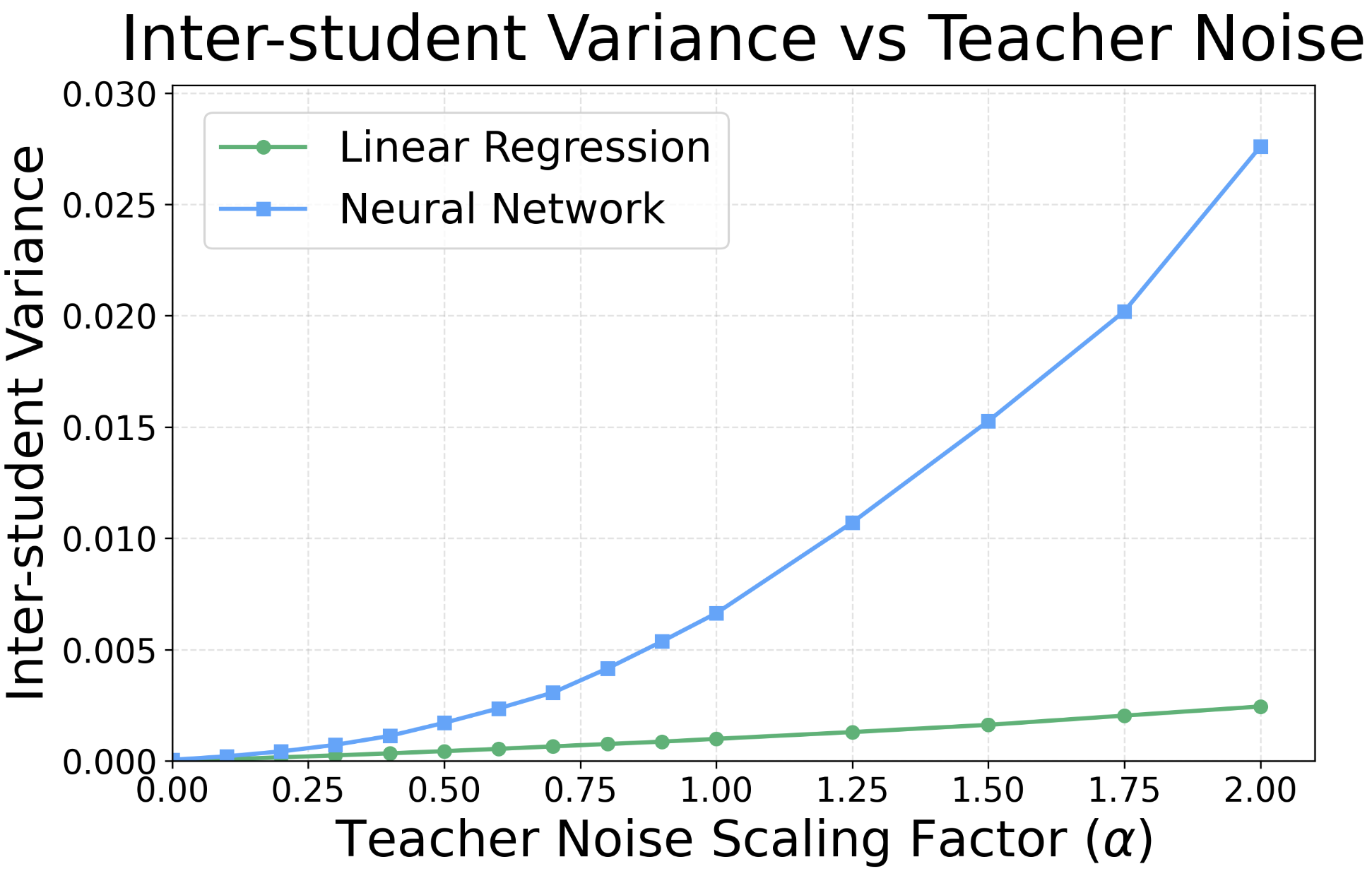}
    \caption{Inter-student variance.}
    \label{fig:simple_uncertainty_1_var}
\end{subfigure}
\caption{
Teacher model output uncertainty in linear regression and neural networks.
}
\label{fig:simple_uncertainty_1}
\end{figure}

\subsection{Teacher Model Output Uncertainty in LLMs}
\label{sec:uncertainty_llm}

Teacher model output uncertainty is especially salient for LLMs. For a single prompt, the teacher model defines a distribution over many plausible responses, yet standard distillation typically compresses this distribution into a single sampled response. This compression fundamentally alters how uncertainty is transferred to the student. 


We distill 10 student models (DistilGPT2~\cite{sanh2019distilbert}) from a GPT2~\cite{radford2019language} teacher on the BioASQ QA benchmark~\cite{krithara2023bioasq} (details in Appendix~\ref{app:datasets_all}). To isolate the effect of teacher output uncertainty, we fix student initialization, vary the teacher model temperature to control output entropy, and evaluate using top-1 decoding for alignment and sampling-based evaluation for uncertainty estimation.



We observe three consistent patterns: \textbf{(1) student–teacher alignment decreases mildly} as teacher model temperature increases, as shown in Figure~\ref{fig:teacher output uncertainty}; \textbf{(2) inter-student variance remains stable} across temperatures (Figure~\ref{fig:teacher output uncertainty}), indicating that teacher sampling entropy alone does not amplify differences between student models when initialization is fixed; and \textbf{(3) student models systematically underestimate predictive variance}, producing distributions that are noticeably narrower than those of the teacher, as illustrated in Figure~\ref{fig:student lower var} (experimental details on predictive variance estimation can be found in Appendix~\ref{app:llm_predictive_distribution}). These behaviors are natural consequences of single-sample supervision: the student is trained on only one draw from a potentially multimodal teacher distribution and therefore learns a single mode rather than the full distribution. We formalize and prove this suppression effect in Appendix~\ref{app:suppression_effect}.





\begin{figure}[t]
  \centering
  \begin{minipage}[t]{0.48\linewidth}
    \centering
    \includegraphics[width=\linewidth]{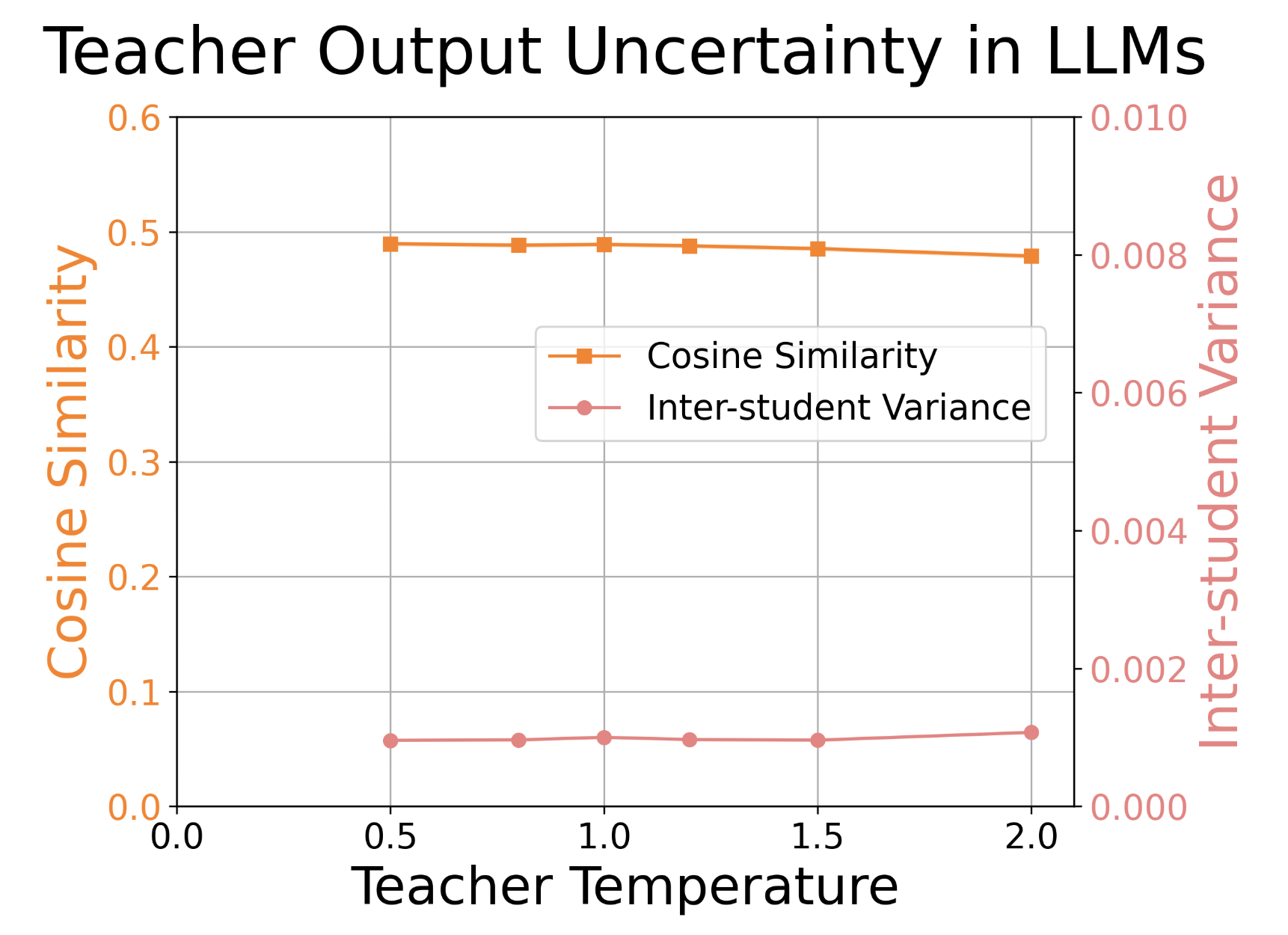}
  \caption{Teacher model output uncertainty in LLMs. 
  }
\label{fig:teacher output uncertainty}
  \end{minipage}%
  \hfill
  \begin{minipage}[t]{0.48\linewidth}
    \centering
    \includegraphics[width=\linewidth]{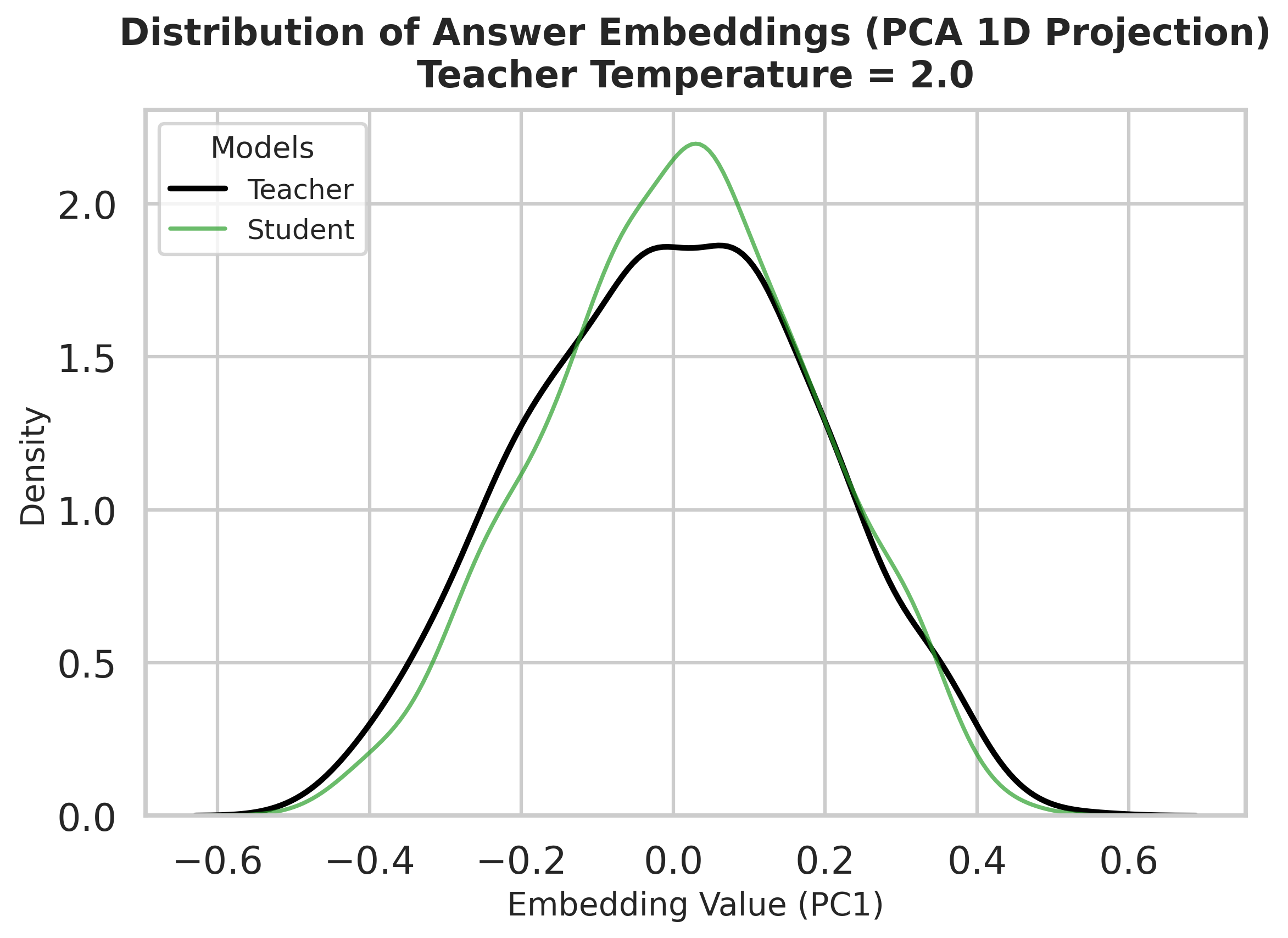}
\caption{Predictive distributions of teacher and student models. 
}
\label{fig:student lower var}
  \end{minipage}
\end{figure}

In summary, teacher model output uncertainty propagates differently in LLMs than in simpler models. While alignment is only modestly affected by teacher model uncertainty, the more important effect is that student models systematically \emph{underestimate} the teacher model’s predictive variance. This results from collapsing a rich teacher distribution into a single sample during training. Consequently, distilled LLMs produce overconfident and narrower predictive distributions, motivating variance-aware distillation methods introduced in Section~\ref{sec:solution}.

\section{Student Model Initialization Uncertainty}
\label{sec:uncertainty_init}

Student model initialization uncertainty captures how much independently trained student models—differing only in their initial parameters or in small perturbations thereof—produce different predictions after knowledge distillation. Below we present a compact, unified treatment that combines a \emph{model-driven} (parameter perturbation) viewpoint with a \emph{data-driven} (bootstrap) viewpoint, summarize theoretical insights in the linear regime, and report concise empirical findings for neural networks and LLMs.

\subsection{Model-Driven Approach: Parameter Perturbations}
\label{sec:init-model-driven}

The model-driven approach probes student model initialization uncertainty by perturbing initial parameters and then retraining or fine-tuning the student model while holding the training data and supervision fixed. Linear regression is deterministic with a closed-form OLS solution and therefore does not exhibit parameter-driven variability. Parameter perturbations are meaningful only for models trained by iterative optimization (e.g., neural networks and LLMs).

For small feed-forward student models trained from scratch under a deterministic teacher model (i.e., teacher model noise $\alpha=0$), we use Kaiming initialization~\cite{he2015delving} and inject multiplicative Gaussian perturbations to the weights, $\tilde{W}_{ij} = W_{ij}(1+\epsilon_{ij})$ with $\epsilon_{ij}\sim\mathcal{N}(0,\sigma_{\mathrm{init}}^2)$, varying $\sigma_{\mathrm{init}}\in[0,0.4]$. Figure~\ref{fig:nn_student_init_uncertainty} summarizes the results: both the mean test MSE and its variance remain largely stable.
This suggests that, for compact feed-forward architectures under a strong (deterministic) distillation signal, small random perturbations of the initial parameters have limited practical effect.


\begin{figure}[t]
  \centering
  \begin{minipage}[t]{0.48\linewidth}
    \centering
    \includegraphics[width=\linewidth]{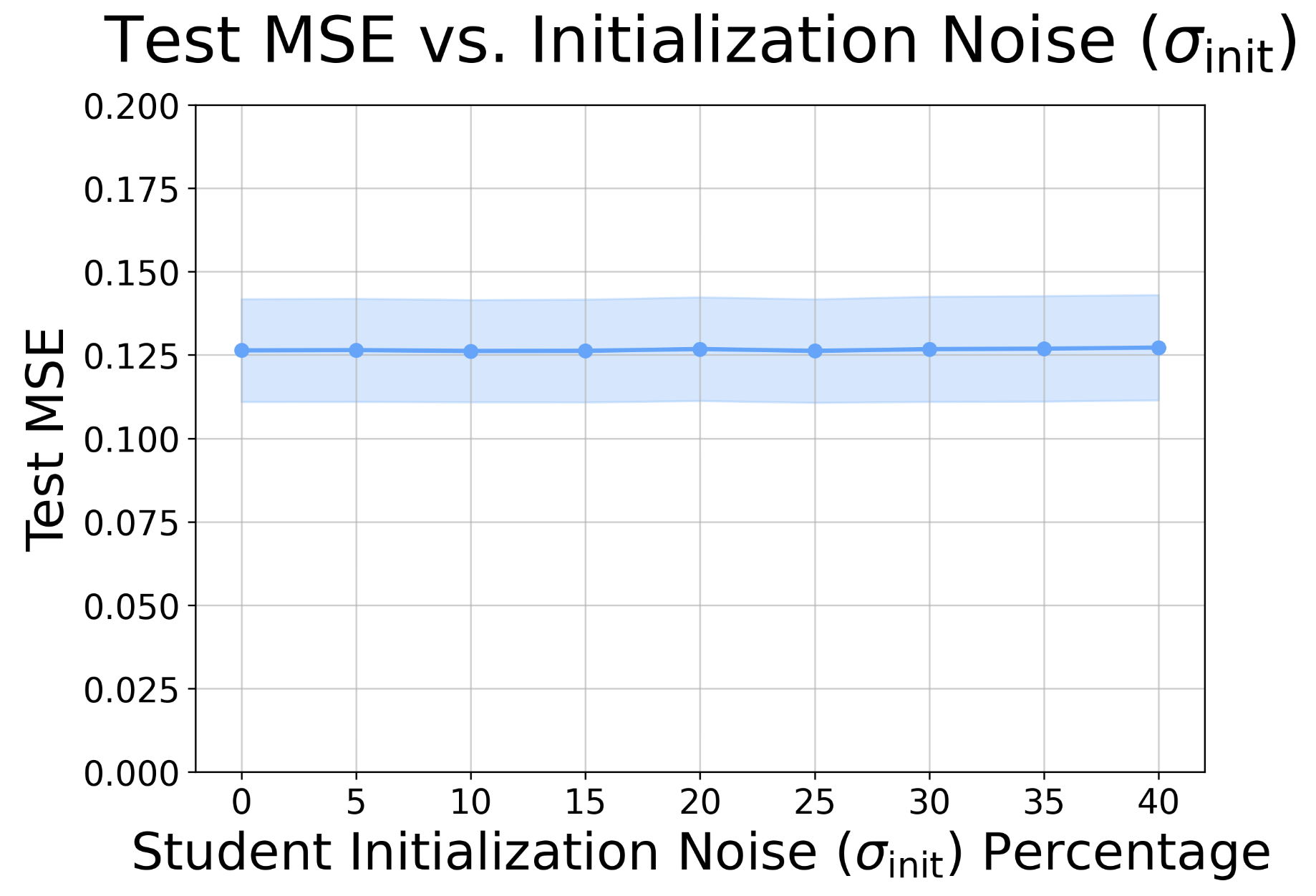}
\caption{Neural network student models: mean test MSE (blue) and inter-student variance (shaded) vs.\ initialization noise.
}
\label{fig:nn_student_init_uncertainty}
  \end{minipage}%
  \hfill
  \begin{minipage}[t]{0.48\linewidth}
    \centering
    \includegraphics[width=\linewidth]{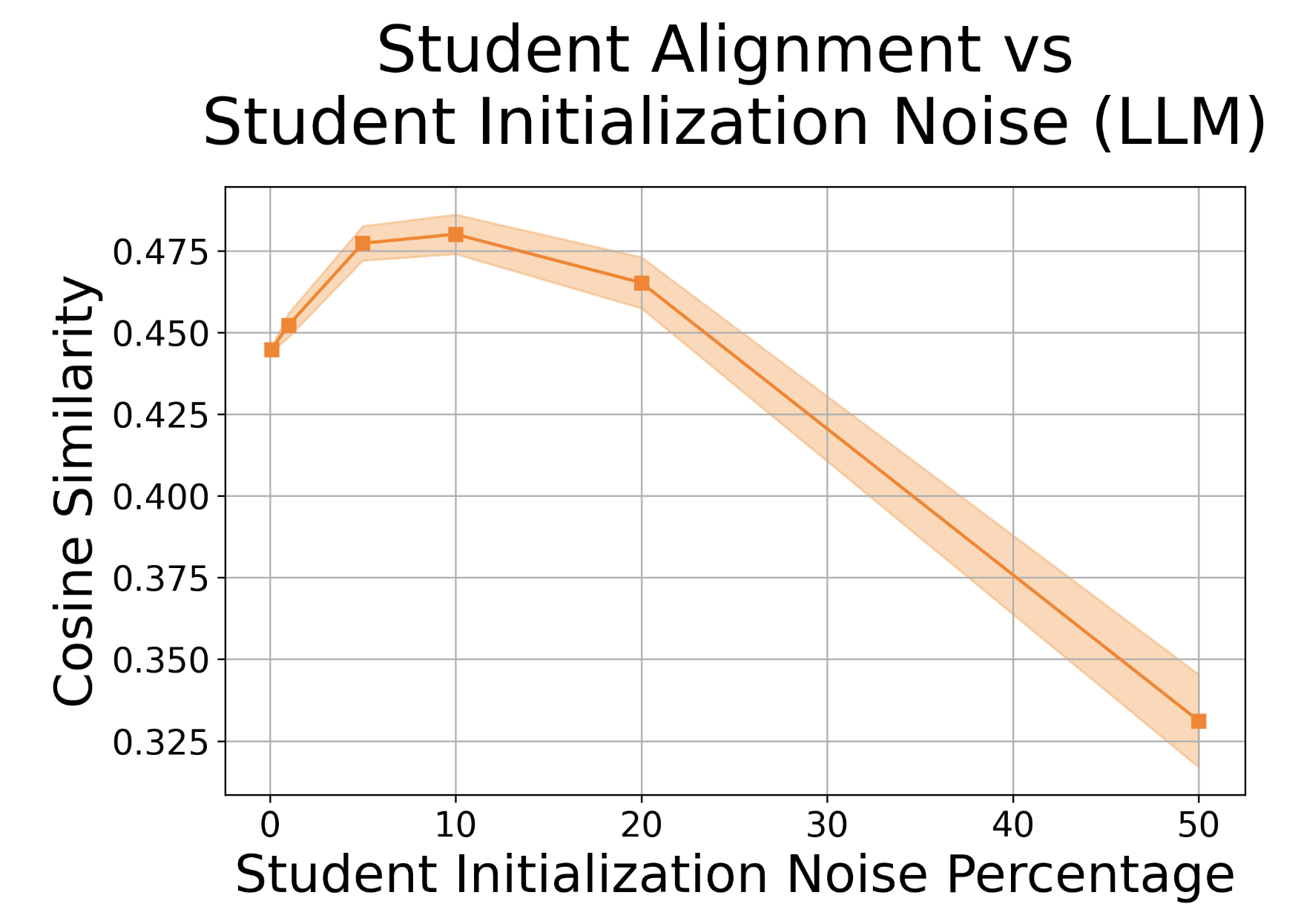}
  \caption{LLM student models: teacher model alignment (orange) and inter-student variance (shaded) vs.\ initialization noise.}
  \label{fig:student_init_uncertainty}
  \end{minipage}
\end{figure}

For pre-trained LLM student models (DistilGPT2 fine-tuned via knowledge distillation from GPT2), we study \emph{local} initialization sensitivity by adding scaled noise to the pre-trained parameters and then fine-tuning. To isolate initialization effects, we fix the teacher model supervision (top-1 teacher model outputs) and evaluate student model alignment and predictive variability. Figure~\ref{fig:student_init_uncertainty} shows that LLM student models are far more sensitive: even small perturbations can reduce alignment with the teacher model and substantially increase inter-student variance. 
This indicates that, for LLMs and relatively weak distillation signals, the initial parameters can determine which basin the optimizer enters, yielding pronounced initialization-driven uncertainty. Additional analysis of optimization trajectories, which provides further insight into this behavior, is deferred to Appendix~\ref{app:param_trajectories}.



In short, the model-driven analysis shows a clear dichotomy: small feed-forward student models are robust to modest initialization perturbations under strong deterministic supervision, whereas large LLM student models are highly sensitive—small local perturbations can redirect fine-tuning toward different optima and markedly increase inter-student variance.

\subsection{Data-Driven Approach: Bootstrap Resampling}
\label{sec:init-bootstrap}

Bootstrap resampling of the training data provides a practical and computationally efficient proxy for probing sensitivity to training data while holding the student model initialization fixed. We study two variants:

\begin{itemize}
  \item \textbf{Teacher-model bootstrap:} construct the teacher-model-labeled dataset $\mathcal{D}_T=\{(x_i,\hat y_{T,i})\}$ with $\hat y_{T,i}=f_T(x_i)$, draw bootstrap replicates $\mathcal{D}_T^{(b)}$, and train student models—all initialized identically—on the bootstrap samples.
  \item \textbf{Ground-truth bootstrap:} draw bootstrap samples $\mathcal{D}^{(b)}$ from the original $(x_i,y_i)$ pairs, and likewise train student models from the same initialization on each replicate.
\end{itemize}

For any test input $x$, the ensemble $\{f_S^{(b)}(x)\}_{b=1}^B$ induces an empirical predictive distribution whose variance provides a purely data-driven measure of student model sensitivity to sampling noise.

\subsubsection{Theoretical results in linear regression}

In OLS regression, the behavior of bootstrap-induced variability admits closed-form characterization:

\begin{theorem}[Degeneracy of teacher-model bootstrap]
\label{thm:teacher_bootstrap_degenerate}
If the teacher model is the OLS estimator $\hat\theta_T=(X^\top X)^{-1}X^\top y$, then for any bootstrap sample of the teacher-labeled dataset, the student-model OLS estimator fitted on the bootstrap replicate equals $\hat\theta_T$. Thus teacher-model bootstrap is degenerate: it induces zero predictive variance.
\qed
\end{theorem}

\begin{theorem}[Ground-truth bootstrap: $1/m$ variance scaling]
\label{thm:groundtruth_bootstrap_variance}
Assume the ground truth dataset is
\(y = X \theta + \eta\), where $\eta\sim\mathcal{N}(0,\sigma^2 I_n)$.
Under standard i.i.d.\ assumptions, the predictive variance from ground-truth bootstrap with sample size $m$ satisfies, for large $m$, $\operatorname{Var}_{\mathrm{boot}}[f_S(x)] \approx \frac{\sigma^2}{m}\,x^\top\Sigma_X^{-1}x$, decaying approximately as $1/m$, where $\Sigma_X$ denotes the second-moment matrix.
\qed
\end{theorem}

\begin{theorem}[Ground-truth bootstrap: expected test MSE]
\label{thm:groundtruth_bootstrap_mse}
Under standard assumptions and feature dimension $d$, the expected test MSE of a ground-truth bootstrap student model trained on $m$ samples behaves as $R_m \approx \sigma^2 + \frac{\sigma^2 d}{m}$, so $R_m\to\sigma^2$ as $m\to\infty$.
\qed
\end{theorem}

\subsubsection{Empirical results in linear and neural student models}

Figure~\ref{fig:lin_bootstrap_mse} confirms the linear-theory predictions on Boston Housing: teacher-model bootstrap produces zero variability, whereas ground-truth bootstrap yields higher variance and mean MSE at small $m$, both decreasing roughly as $1/m$.  
For neural student models (Figure~\ref{fig:nn_bootstrap_mse}), both variants induce non-trivial variability: increasing $m$ reduces both mean MSE and its spread, and ground-truth bootstrap yields larger variance than teacher-model bootstrap at small $m$.  
These findings demonstrate that bootstrap resampling provides a useful data-driven probe of student-model sensitivity in nonlinear settings.

\begin{figure}[t]
  \centering
  \begin{minipage}[t]{0.48\linewidth}
    \centering
    \includegraphics[width=\linewidth]{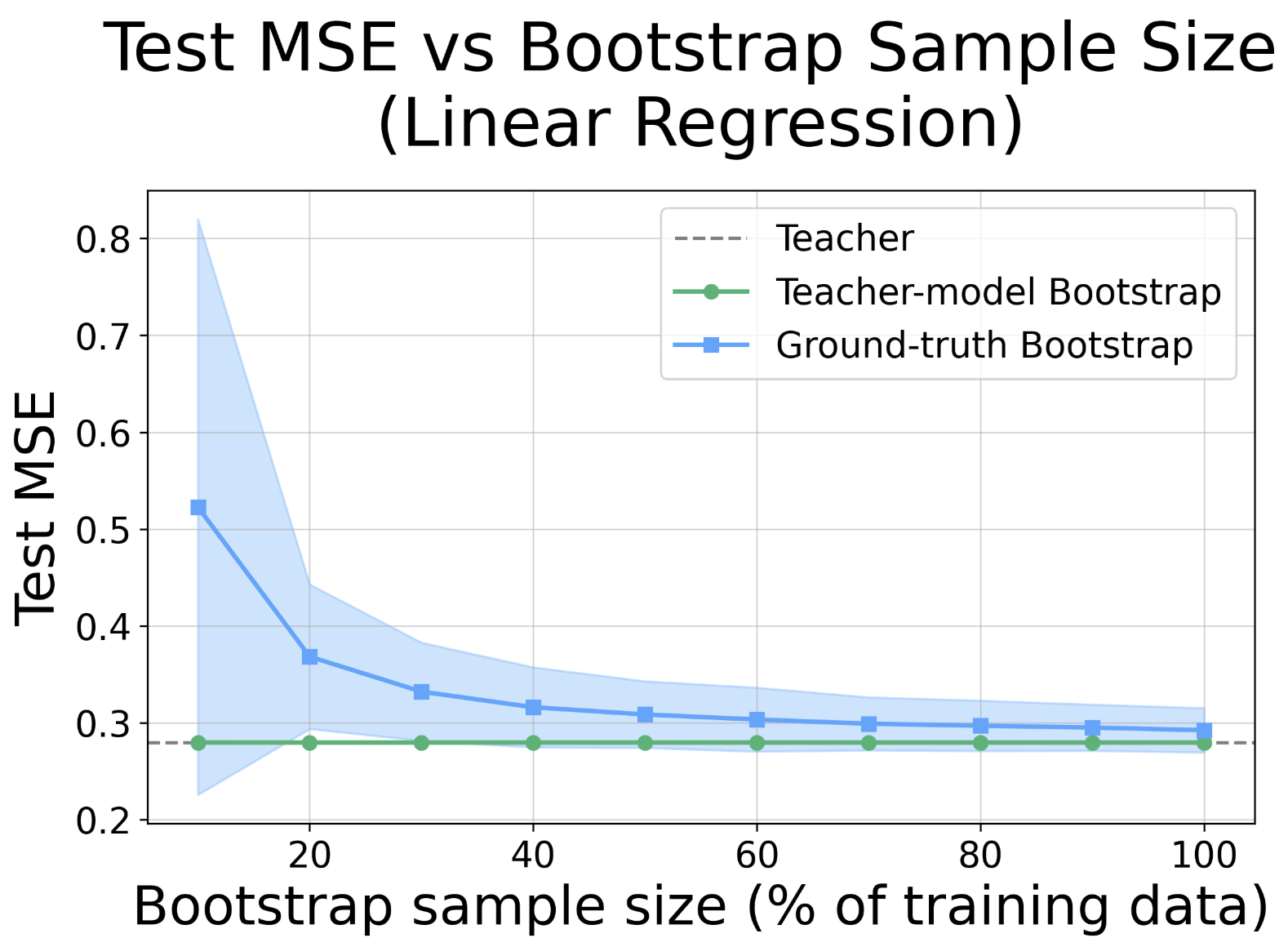}
    \caption{Data-driven student initialization uncertainty in linear regression.
    }
    \label{fig:lin_bootstrap_mse}
  \end{minipage}%
  \hfill
  \begin{minipage}[t]{0.48\linewidth}
    \centering
    \includegraphics[width=\linewidth]{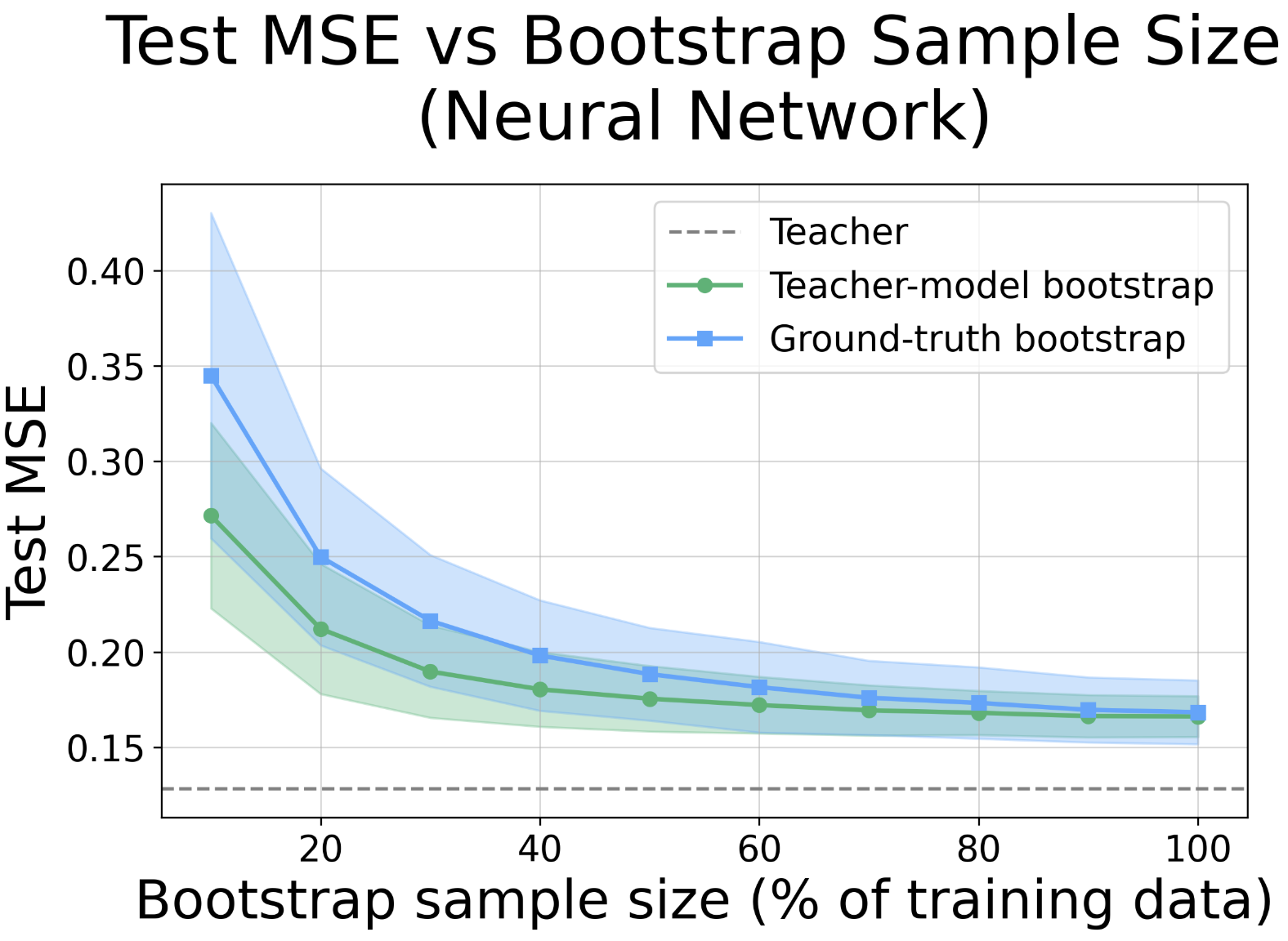}
    \caption{Data-driven student initialization uncertainty in neural networks.
    }
    \label{fig:nn_bootstrap_mse}
  \end{minipage}
\end{figure}

\subsubsection{Summary}
Teacher-model bootstrap induces no student-model variability and becomes uninformative when the teacher model is exactly realizable by the student-model class. Ground-truth bootstrap, by contrast, reveals finite-sample variability that decays as $1/m$ and thus provides a practical, low-cost proxy for initialization sensitivity in nonlinear or misspecified regimes.

\subsection{Synthesis and Practical Recommendations}

Combining model-driven and data-driven perspectives yields practical guidance. First, for \textbf{small, well-behaved student models} (e.g., compact feed-forward networks) trained under strong deterministic supervision, initialization perturbations have limited effect; a single distillation run often suffices. Second, for \textbf{large, high-capacity student models} (e.g., LLMs), initialization (or small perturbations around it) can substantially alter fine-tuning outcomes. Practitioners should average over multiple initializations or employ variance-aware distillation procedures. Last, when computational resources are constrained, \textbf{ground-truth bootstrap} with a fixed initialization offers a lightweight proxy for estimating student-model sensitivity. Teacher-model bootstrap is informative only when the teacher model is not exactly realizable by the student-model class.

Formal proofs of the linear-theory results appear in Appendix~\ref{app:proofs}. Experimental details (architectures, optimization hyperparameters, dataset splits) are provided in Appendix~\ref{app:bootstrap_setup}.

\section{Student Model Output Uncertainty}
\label{sec:uncertainty_output}

We now examine the third source of uncertainty in knowledge distillation: \emph{student model output uncertainty}, which reflects the predictive variability inherent in a single student model. As noted earlier, linear regression is fully deterministic once trained and therefore does not exhibit this form of uncertainty. To study output uncertainty in a controlled way, we use classification tasks where predictive randomness naturally arises from softmax outputs.

\subsection{Logistic Regression and Neural Networks}
\label{sec:simple_student_output_uncertainty}


We begin with a simple distillation experiment on the Digits dataset~\cite{optical_recognition_of_handwritten_digits_80} (details in Appendix~\ref{app:datasets_all}).
The teacher model is trained using ground-truth labels. For each input $x$, the teacher produces a predictive distribution $p_T(y \mid x)$; the student model is trained on hard pseudo-labels $\hat{y}^{(T)}(x)=\arg\max_y p_T(y \mid x)$. 

To quantify model output uncertainty, we compute predictive entropy $H(x)=-\sum_{y} p(y \mid x)\log p(y\mid x)$ for each test example. For both teacher and student models, we report the mean and standard deviation of entropy across the full test set, as well as separately for correctly and incorrectly classified samples. We perform this experiment with both logistic regression and a small neural network. Additional implementation details can be found in Appendix~\ref{app:student_output_uncertainty_extra_datasets}.

Figures~\ref{fig:log_reg_student_output_uncertainty} and~\ref{fig:nn_student_output_uncertainty} summarize the results. In both model families, the student model closely matches the teacher model’s accuracy and predictive entropy. Entropy differences between teacher and student models are small, and the mean \(\pm\) standard deviation intervals overlap in all groups. This indicates that, for these supervised classification tasks, the student model largely preserves the teacher model’s predictive uncertainty.

\begin{figure}[t]
\centering
\begin{subfigure}[t]{0.48\linewidth}
    \centering
    \includegraphics[width=\linewidth]{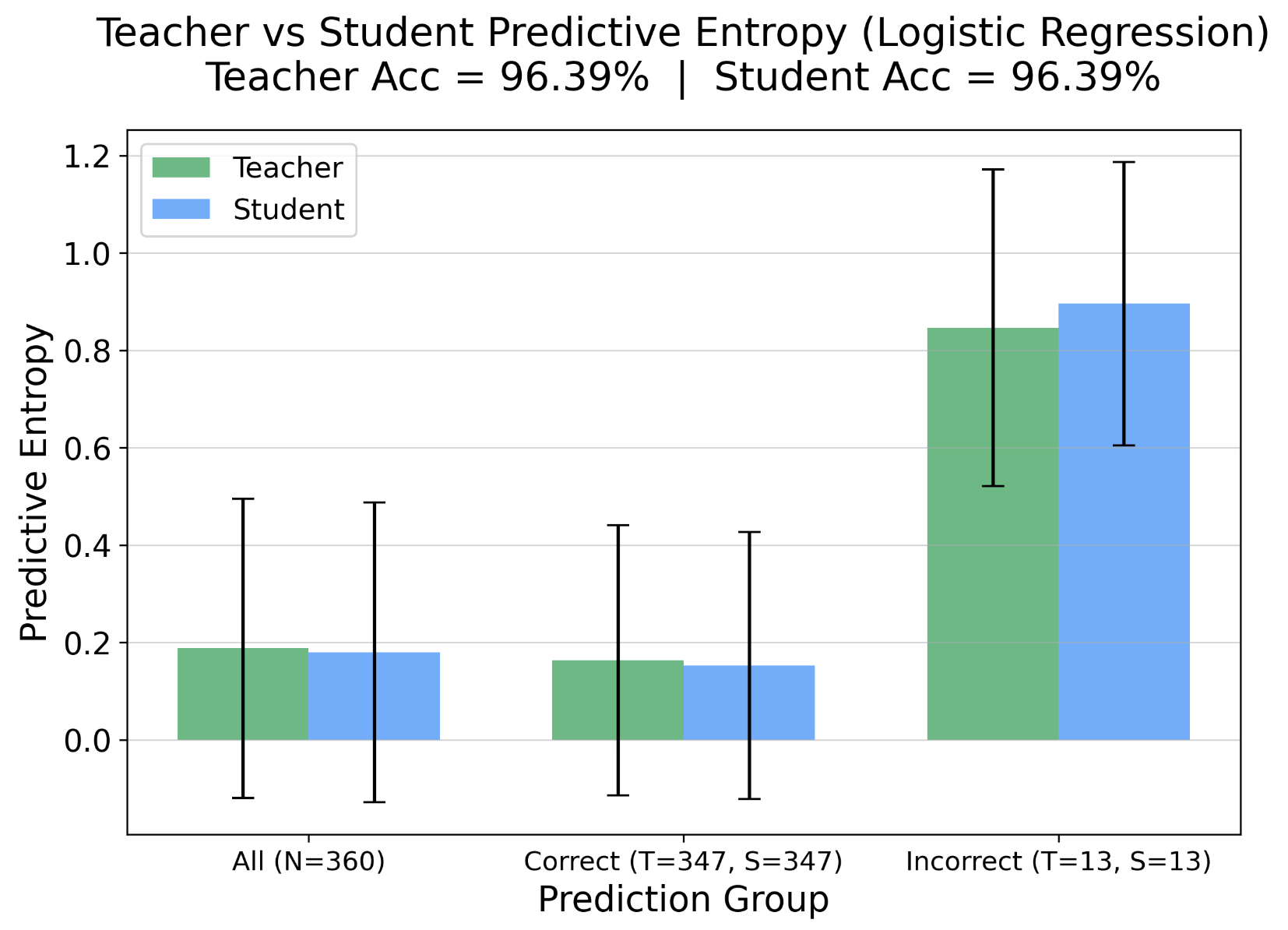}
    \caption{Logistic regression.}
    \label{fig:log_reg_student_output_uncertainty}
\end{subfigure}
\hfill
\begin{subfigure}[t]{0.48\linewidth}
    \centering
    \includegraphics[width=\linewidth]{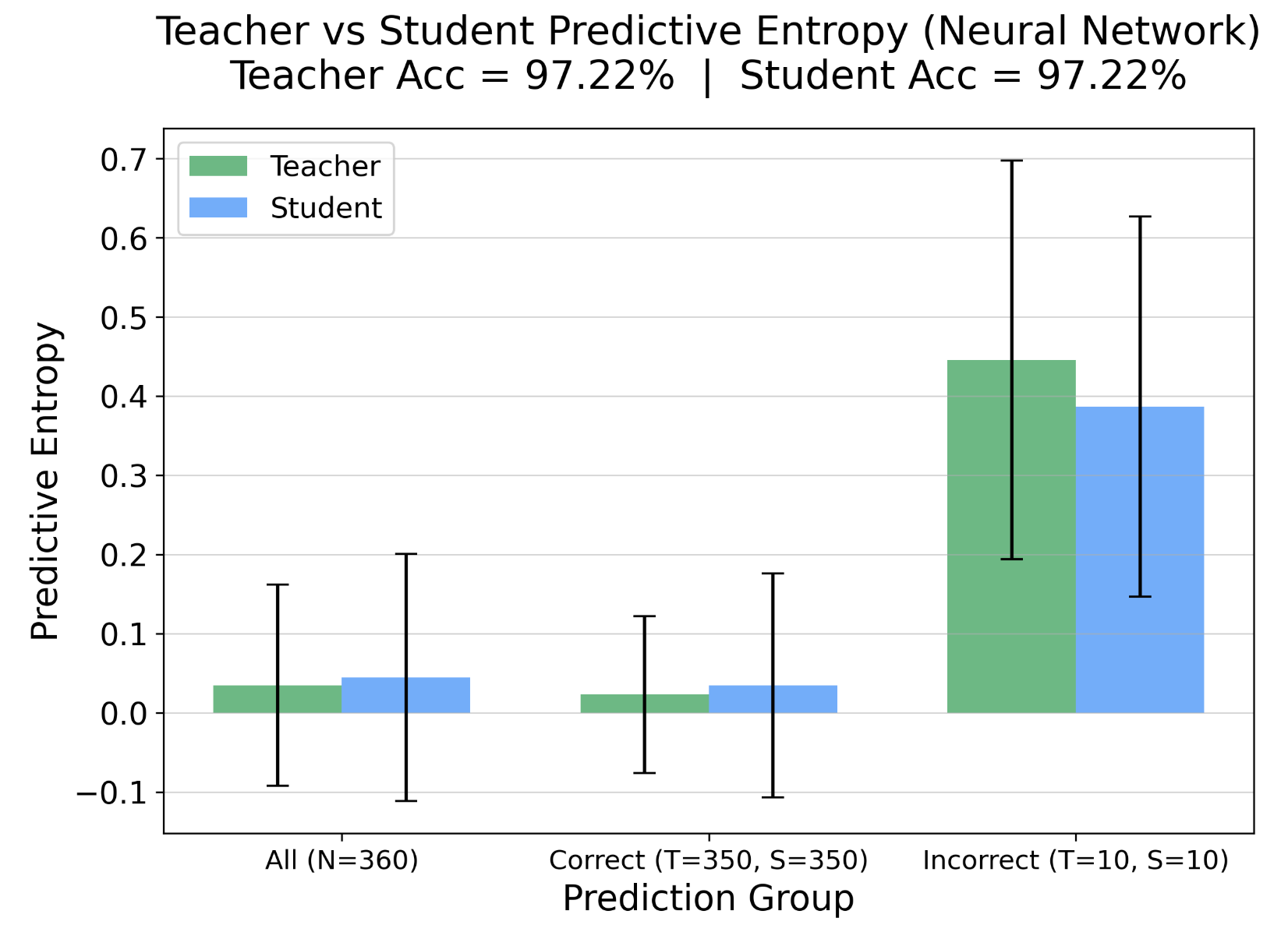}
    \caption{Neural networks.}
    \label{fig:nn_student_output_uncertainty}
\end{subfigure}
\caption{
Student model output uncertainty in classification.}
\end{figure}

We observe the same pattern across several additional datasets--including Wine, Breast Cancer, MNIST, and Covertype—where teacher and student model entropies remain closely aligned. Detailed numerical results are provided in Appendix~\ref{app:student_output_uncertainty_extra_datasets}.

\subsection{LLMs: Consistent Distillation Yields Consistent Students}

LLMs remain stochastic even when initialization and supervision are fixed since generative decoding samples from a predictive distribution. Ideally, if the teacher model outputs and the student model initialization are identical across runs, independently distilled student models should produce similar predictive distributions.

To isolate output uncertainty in this setting, we fix all sources of randomness except the inherent stochasticity of generation. We distill 20 student models from a single teacher model, all using the same teacher model outputs and identical parameter initialization. At inference time, we sample 20 responses from the teacher model and each student model. For each model $\mathcal{M}$, we compute the mean embedding vector $m_{\mathcal{M}}$ and variance vector $v_{\mathcal{M}}$ across its responses. To assess alignment, we form $\mathcal{A}=\{\text{CosSim}(m_T,m_1),\ldots,\text{CosSim}(m_T,m_{20})\}$, where $\mathcal{A}$ captures how closely each student model resembles the teacher model's output.

\begin{figure}[t]
\centering
\begin{subfigure}[t]{0.48\linewidth}
    \centering
    \includegraphics[width=0.88\linewidth]{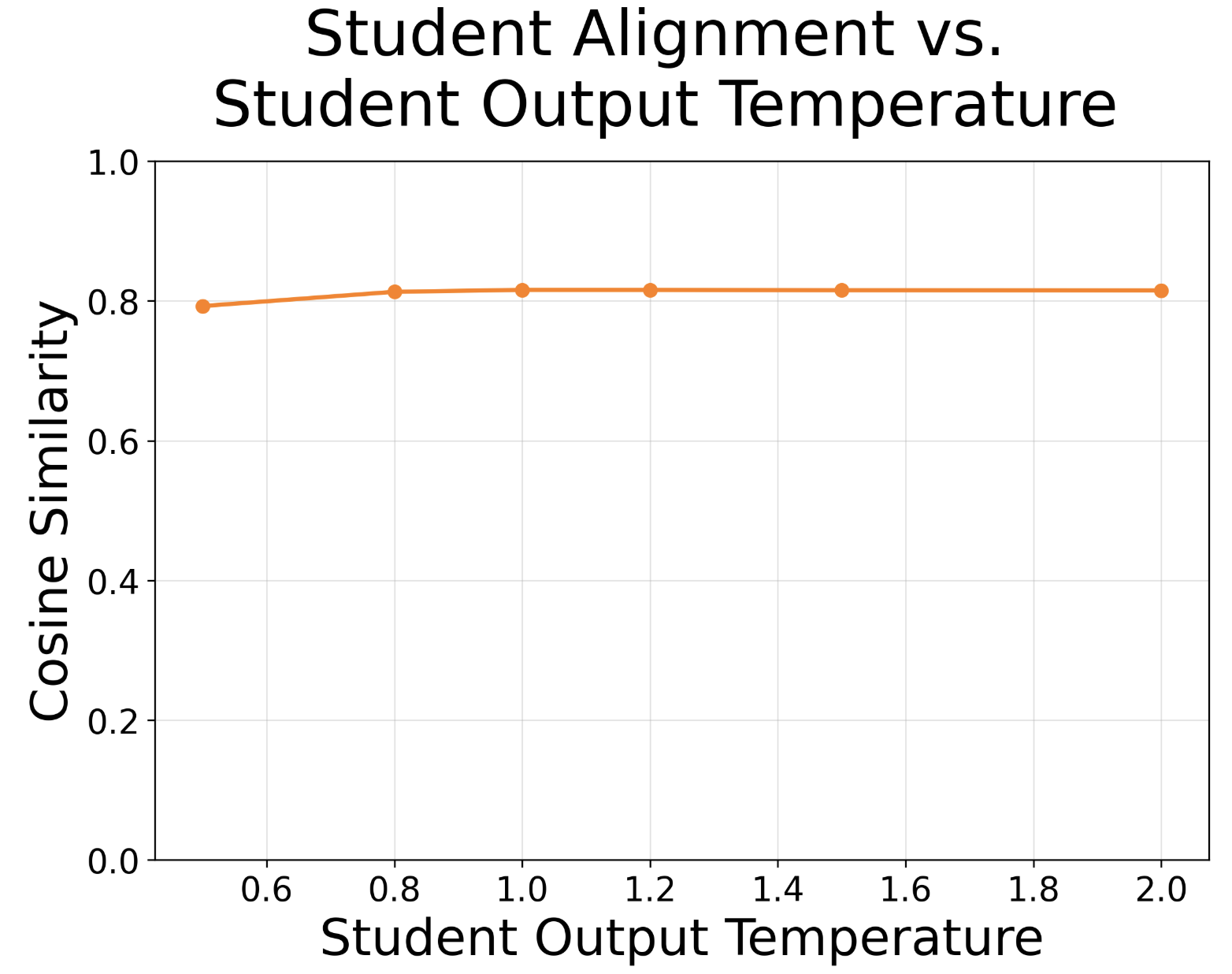}
    \caption{Student alignment (cosine similarity) vs.\ student noise.}
    \label{fig:student_output_uncertainty_a}
\end{subfigure}
\hfill
\begin{subfigure}[t]{0.48\linewidth}
    \centering
    \includegraphics[width=0.95\linewidth]{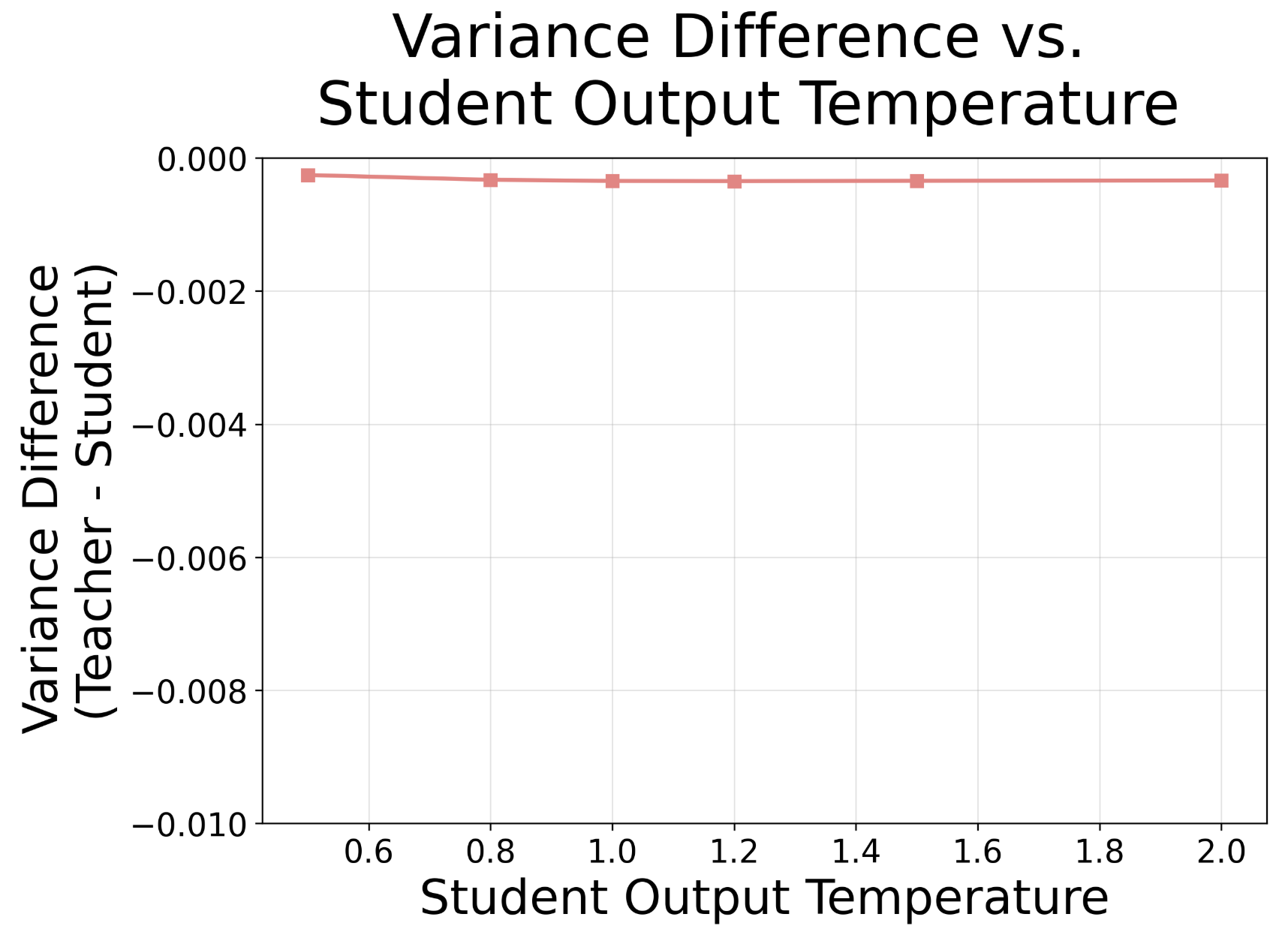}
    \caption{Predictive variance difference vs.\ student noise.}
    \label{fig:student_output_uncertainty_b}
\end{subfigure}
\caption{Student model's sensitivity to student output uncertainty in LLMs.
}
\label{fig:student_output_uncertainty}
\end{figure}

Figure~\ref{fig:student_output_uncertainty_a} shows that distilled student models exhibit high alignment with the teacher and low inter-student variance (shaded region too small to be visible). When supervision is fixed, distillation consistently produces student models with similar predictive distributions, indicating that the remaining output uncertainty is primarily due to generative sampling rather than training instability.

\subsubsection{Comparison of intra-model variance in student and teacher models}
\label{sec:llm_student_uncertainty_intra_model}

Although the mean predictive embeddings align closely, we observe small differences in predictive variance. For each student model $S$, we compute $\Delta v(S) = \frac{1}{d}\sum_{j=1}^d (v_{T,j}-v_{S,j})$, where $d$ is the embedding dimensionality. In Figure~\ref{fig:student_output_uncertainty_b}, across teacher temperatures, the values of $\Delta v(S)$ are consistently negative, indicating that student models display slightly higher intra-model variance than teacher model. The differences, however, are numerically small.

Overall, the discrepancy in variance is negligible, and teacher and student models can be regarded as having comparable output uncertainty in this setting. When the supervision signal is consistent, student models inherit a stable predictive distribution and do not introduce significant additional uncertainty.

\section{How to Correct Student Model Uncertainty?}
\label{sec:solution}

The analyses above show that a substantial portion of the uncertainty mismatch in knowledge distillation arises from two factors: noise in teacher model outputs and distillation signals that are too weak to correct small differences in student model initialization. In this section, we present two simple and effective strategies to reduce inter-student uncertainty and preserve intra-student uncertainty during distillation. The first, \emph{Averaging}, constructs lower-variance supervision by using the sample mean of multiple teacher responses. The second, \emph{Variance-Weighting}, forms an optimal combination of teacher and student model estimates based on their respective variances. We first analyze these methods in linear regression, where theory is exact, and then demonstrate that the same principles extend to nonlinear models.

Both strategies aim to reduce the effective noise in the targets used for distillation. Averaging reduces teacher model output variance by a factor of \(1/k\) when \(k\) independent teacher model responses are used per input. Variance-weighting further refines this approach by weighting teacher and student model estimates according to their inverse variances, yielding a minimum-variance combination.

\subsection{Theoretical Analysis on Simple Models}
\label{sec:theory_simple}

Because the intrinsic teacher model variance \(\sigma_T^2\) cannot be altered directly, our goal is to construct lower-variance distillation targets that yield more stable student models. In the linear regression setting, both averaging and variance-weighting lead to student model predictions with provably reduced inter-student variance. Experiments with linear regression and neural networks confirm that these reductions hold broadly in practice.

\subsubsection{Averaging teacher model responses}

Suppose the teacher model produces independent noisy responses $y_{i,j}^{(T)} = f_T(x_i) + \varepsilon_{i,j}$, where $\varepsilon_{i,j} \overset{\text{i.i.d.}}{\sim} \mathcal{N}(0,\sigma_T^2)$. Averaging $k$ responses gives $\mu_{T,i} = \frac{1}{k}\sum_{j=1}^{k} y_{i,j}^{(T)}$, which satisfies $\operatorname{Var}(\mu_{T,i}) = \sigma_T^2/k$. Training the student model on these averaged targets reduces both the noise in the distillation dataset and the inter-student predictive variance.

\begin{theorem}
\label{thm:averaging_var}
In linear regression distillation, dataset noise \(\operatorname{Var}(\mu_T)\), variance of the student model parameter estimator \(\operatorname{Var}(\widehat{\theta}_S)\), and the inter-student predictive variance \(\operatorname{Var}(\widehat{f}_S(x))\) all decay at order \(1/k\) as the number of teacher model responses per input increases.
\qed
\end{theorem}

\begin{theorem}
\label{thm:averaging_umvue}
The sample mean \(\mu_{T,i}\) is the uniformly minimum-variance unbiased estimator of the teacher model output \(f_T(x_i)\).
\qed
\end{theorem}

\subsubsection{Variance-Weighting}


When the student model itself is stochastic—as is the case for LLMs—both the teacher and the student produce random outputs with associated variances. The student model can use its own predictive variance to decide how much to trust the teacher model.

Let $\widehat{\sigma_{T,i}^2}$ be an empirical estimate of the teacher model variance and let $\sigma_{S,i}^2$ be the student model variance for input $x_i$. A variance-aware correction constructs the combined target $\widehat{y}(x_i) = w_{T,i}\,\mu_{T,i} + w_{S,i}\,\mu_{S,i}$, with $w_{T,i} + w_{S,i} = 1$.

\begin{theorem}
\label{thm:var_weight_optimal}
The weights that minimize the variance of the target $\widehat{y}(x_i)$ are
$w_{T,i} = \frac{1/\widehat{\sigma_{T,i}^2}}{1/\widehat{\sigma_{T,i}^2} + 1/\sigma_{S,i}^2}$ and
$w_{S,i} = \frac{1/\sigma_{S,i}^2}{1/\widehat{\sigma_{T,i}^2} + 1/\sigma_{S,i}^2}$.
\qed\end{theorem}

\begin{theorem}
\label{thm:var_weight_properties}
The minimum achievable variance satisfies
$\operatorname{Var}[\widehat{y}(x_i)]_{\min}
= \frac{1}{k}\cdot \frac{\widehat{\sigma_{T,i}^2}\sigma_{S,i}^2}{\widehat{\sigma_{T,i}^2} + \sigma_{S,i}^2}$,
and enjoys the following properties:
\begin{enumerate}
    \item it is always smaller than the variance from using either model alone,
    \item it downweights a noisy teacher model when $\widehat{\sigma_{T,i}^2}$ is large,
    \item it downweights a noisy student model when $\sigma_{S,i}^2$ is large, and
    \item it halves the variance when the two variances are equal. \qed
\end{enumerate}
\end{theorem}

\subsection{Empirical Evidence in Linear Regression and Neural Networks}

Figures~\ref{fig:linear_regression} and~\ref{fig:neural_network} show the effects of averaging and variance-weighting in linear regression and neural networks using the Boston Housing dataset~\cite{harrison1978hedonic}. In both settings, increasing the number of teacher model responses steadily reduces student model variance. Moreover, even using two teacher model responses provides a large improvement over single-response distillation. Variance-weighting consistently yields the most stable student models.

These results confirm that the theoretical benefits of averaging and variance-weighting extend beyond the linear regime.

\begin{figure}[t]
\centering
\begin{subfigure}[t]{0.48\linewidth}
    \centering
    \includegraphics[width=\linewidth]{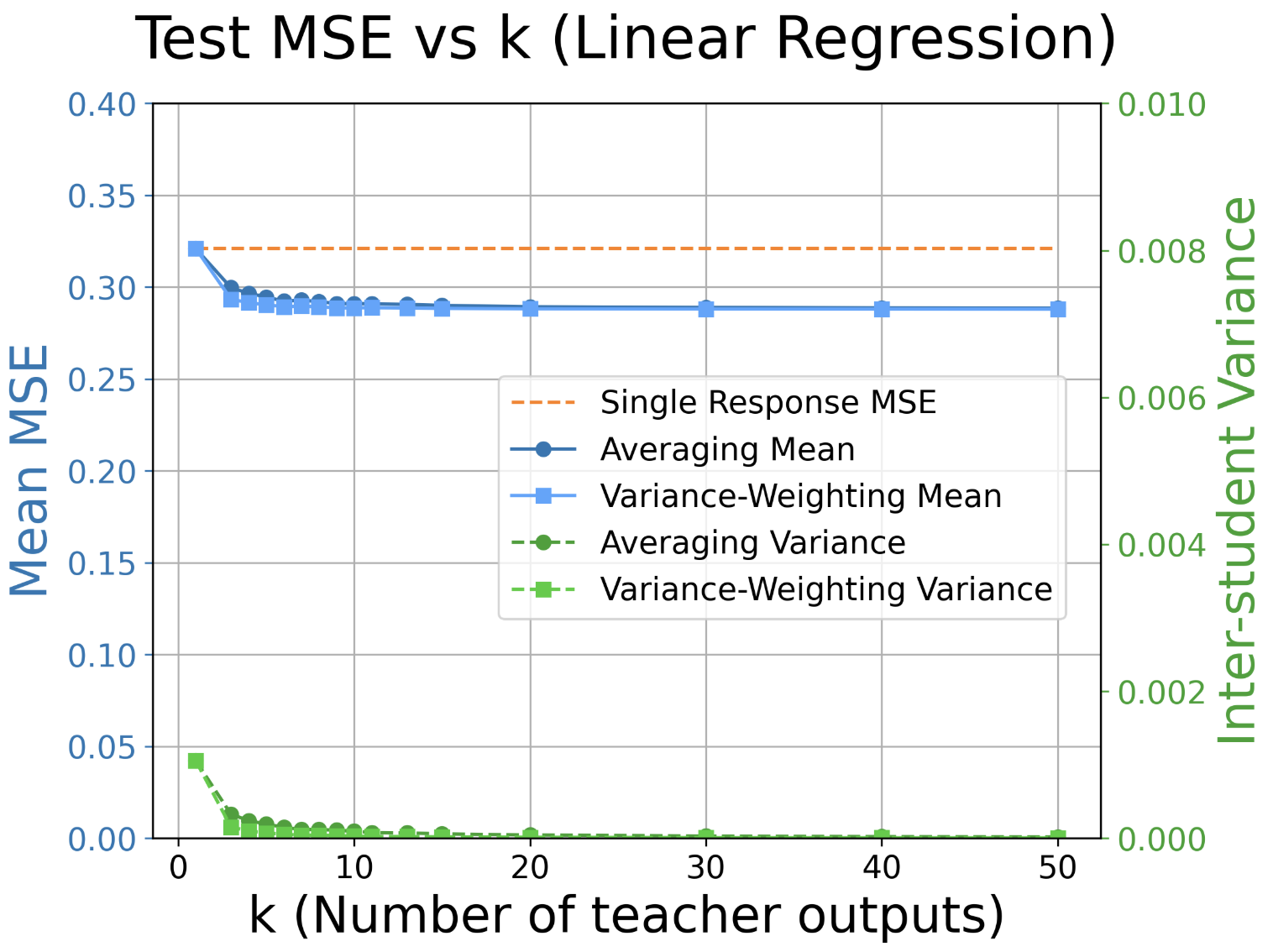}
    \caption{Linear regression.}
    \label{fig:linear_regression}
\end{subfigure}
\hfill
\begin{subfigure}[t]{0.48\linewidth}
    \centering
    \includegraphics[width=\linewidth]{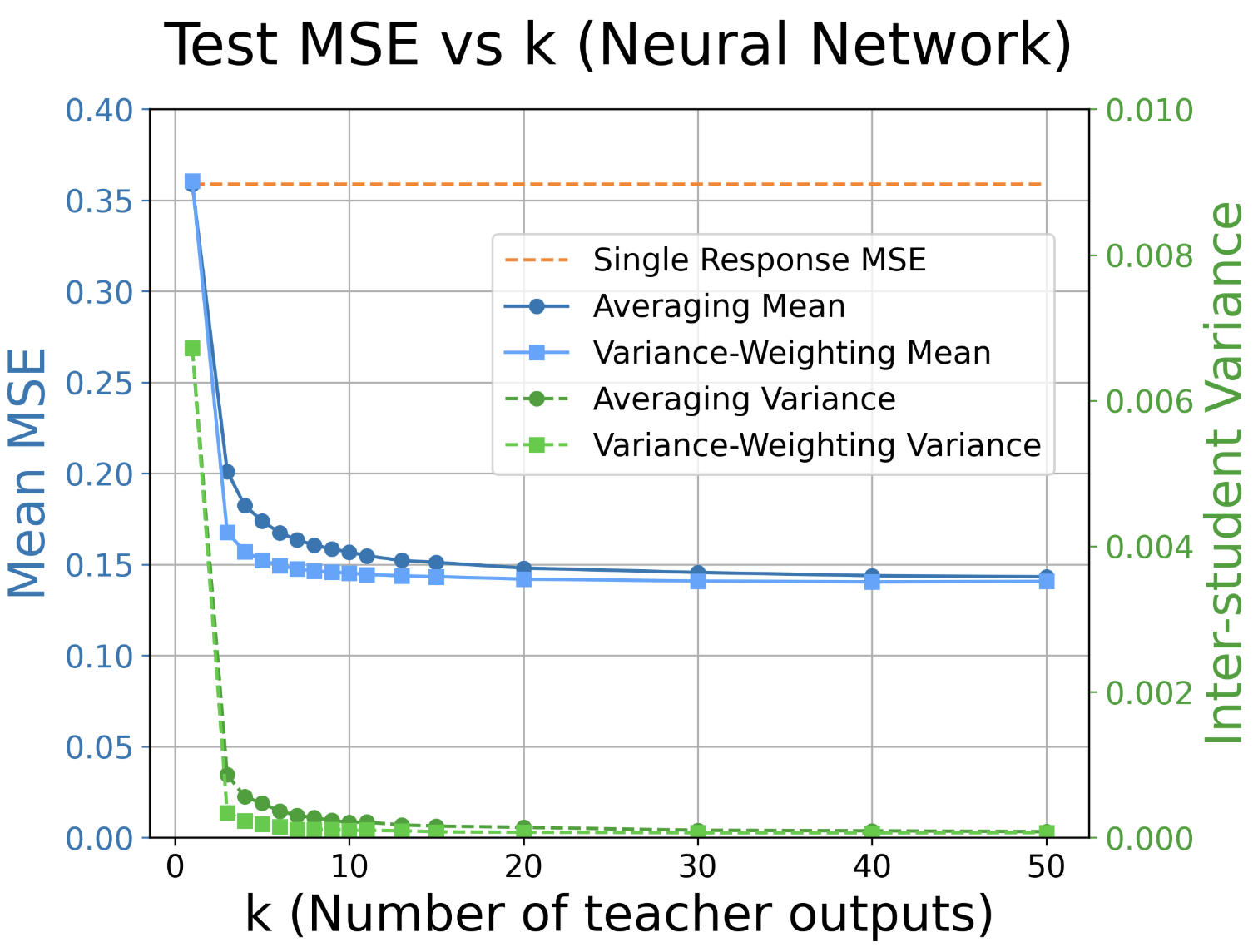}
    \caption{Neural networks.}
    \label{fig:neural_network}
\end{subfigure}
\caption{
Effectiveness of averaging and variance-weighting approaches in (a) linear regression and (b) neural network distillations. The blue curves represent the mean test MSE, the green curves depict the corresponding inter-student variance, and the orange line denotes the standard distillation pipeline with one teacher output. 
}
\label{fig:avg_var_effectiveness_simple}
\end{figure}

\subsection{Extending Methods to LLMs}
\label{sec:reduce_noise_llm}

The same variance-aware principles extend naturally to the distillation of LLMs. Unlike scalar regression targets, natural language responses are high-dimensional and structured, encoding substantially richer information. Sampling multiple responses from a teacher LLM therefore provides a more faithful empirical approximation of the teacher’s predictive distribution. We use this multi-response setting as an additional baseline for LLM distillation.

Under standard single-response distillation, the student is trained on a single teacher-generated sequence $\mathbf{y}^*$. In contrast, multi-response distillation minimizes the objective $\mathcal{L}_{\mathrm{LLM}}^{(k)} = -\frac{1}{k}\sum_{i=1}^{k} \log p_S(\mathbf{y}^{*(i)} \mid \mathbf{x})$, where $\{\mathbf{y}^{*(i)}\}_{i=1}^k$ are samples from the teacher model. This objective more closely approximates the ideal, but intractable, distribution-matching formulation by exposing the student to multiple modes of the teacher’s output distribution.

As shown in Table~\ref{tab:noise_methods}, increasing the number of teacher responses consistently improves alignment with ground-truth answers and reduces hallucination. Among the evaluated strategies, variance-weighting yields the largest and most consistent gains.

We further evaluate how well student models inherit the teacher’s noise structure, where systematic noise serves as a proxy for reducible epistemic uncertainty introduced by the distillation process (see Appendix~\ref{app:noise_exps}). Both averaging and variance-weighting reduce systematic noise and increase the correlation between student and teacher variability, with variance-weighting again providing the strongest alignment.

Overall, these results demonstrate that leveraging multiple teacher responses—particularly when combined with variance-weighting—offers a simple and effective mechanism for reducing inter-student uncertainty in LLM distillation. The resulting student models more faithfully capture the teacher’s mean behavior while exhibiting substantially improved stability.

\begin{table}[t]
\centering
\tiny
\caption{
Comparison of noise characteristics for LLM students distilled using multiple teacher responses. $R^2$ denotes the correlation between the teacher's and student's noise.
}
\label{tab:noise_methods}
\begin{tabular}{cc|cc|ccc}
\toprule
 &  & \multicolumn{2}{c|}{\textbf{Cosine Sim. to Ground Truth}} 
 & \multicolumn{3}{c}{\textbf{Noise from Knowledge Distillation}} \\
\cmidrule(lr){3-4} \cmidrule(lr){5-7}
\textbf{$k$} & \textbf{Method} 
& \textbf{Mean} & \textbf{Inter-Stud. Var.} 
& $\mathbf{R^2}$ & \textbf{Avg. Noise} & \textbf{Avg. System. Noise} \\
\midrule
-- & Direct FT & -- & -- & 0.151 & 0.505 & -- \\
\midrule
1 & Single Response & 0.319 & 0.015 & 0.071 & 0.678 & 0.173 \\
\midrule
\multirow{3}{*}{3} 
 & Multi-Response & 0.346 & 0.016 & 0.109 & 0.663 & 0.158 \\
 & Averaging & 0.352 & 0.016 & 0.118 & 0.655 & 0.150 \\
 & Variance-Weighting & 0.353 & 0.016 & 0.129 & 0.658 & 0.153 \\
\midrule
\multirow{3}{*}{5} 
 & Multi-Response & 0.353 & 0.017 & 0.138 & 0.657 & 0.152 \\
 & Averaging & 0.365 & 0.017 & 0.144 & 0.641 & 0.136 \\
 & Variance-Weighting & 0.365 & 0.017 & 0.152 & 0.644 & 0.139 \\
\bottomrule
\end{tabular}
\end{table}

\section{Discussion and Conclusion}
\label{sec:con}

This paper examined how uncertainty propagates through sequence-level knowledge distillation and identified systematic mismatches between teacher and student models. While knowledge distillation is effective for compression and accuracy, standard single-response distillation suppresses intra-student uncertainty, leading students to underestimate teacher predictive entropy and produce overly narrow distributions. At the same time, knowledge distillation does not reliably eliminate inter-student uncertainty: small differences in initialization can result in persistent variability across independently distilled students.

We show that these issues can be addressed with simple variance-aware strategies. Averaging multiple teacher responses reduces distillation noise at rate $\mathcal{O}(1/k)$, while variance-weighting yields minimum-variance combinations of teacher and student estimates. These methods are theoretically optimal in linear models and empirically effective in neural networks and LLMs, consistently improving alignment and reducing systematic noise.

Overall, our results reframe knowledge distillation as an \emph{uncertainty transformation} rather than mere compression. Students may match teachers on benchmarks yet behave differently when uncertainty matters, risking overconfidence or hallucination in high-stakes settings. Future work should explore uncertainty across scales and domains and design knowledge distillation objectives that balance fidelity with uncertainty preservation, enabling distilled models to remain both accurate and appropriately uncertain.

\printbibliography

\appendix

\section{Mathematical Proofs}
\label{app:proofs}

\subsection{Proof of Theorem~\ref{thm:inter_student_variance}}

\begin{proof}
OLS theory gives
\[
\operatorname{Var}(\widehat{\theta}_S) = \sigma_T^2(X^\top X)^{-1}.
\]
For any test input $x_\star$,
\[
\operatorname{Var}\big[\widehat{f}_S(x_\star)\big]
= \sigma_T^2\, x_\star^\top (X^\top X)^{-1} x_\star,
\]
which, when averaged over test inputs, yields a quantity linear in $\sigma_T^2$.
\end{proof}

\subsection{Proof of Theorem~\ref{thm:teacher_output_mse}}

\begin{proof}
First, we write the OLS estimator using the noisy teacher labels:
\begin{equation}
\label{eq:student model parameter}
    \begin{split}
        \widehat\theta_S=(X^\top X)^{-1}X^\top y^{(T)}
    =\theta_T + (X^\top X)^{-1}X^\top\varepsilon
    \end{split}
\end{equation}
For a test input $x_i$, the prediction error relative to the teacher mapping is
\[
x_i^\top\widehat\theta_S - x_i^\top\theta_T
    = x_i^\top (X^\top X)^{-1}X^\top\varepsilon.
\]
Stacking all test inputs gives the error vector
\[
\Delta = X_{\mathrm{test}}(\widehat\theta_S-\theta_T)
    = X_{\mathrm{test}}(X^\top X)^{-1}X^\top\varepsilon.
\]
Thus the evaluation MSE is
\begin{equation*}
    \begin{split}
        \mathcal{L}_{\mathrm{eval}}^{(T)} = \frac{1}{n}\,\Delta^\top\Delta
    = \frac{1}{n}\,\varepsilon^\top X (X^\top X)^{-1} X_{\mathrm{test}}^\top 
    X_{\mathrm{test}} (X^\top X)^{-1} X^\top \varepsilon
    \end{split}
\end{equation*}

Taking expectation over $\varepsilon$, using 
$\mathbb{E}[\varepsilon\varepsilon^\top]=\sigma_T^2 I_n$ and the cyclic property of trace,
\[
\mathbb{E}[\mathcal{L}_{\mathrm{eval}}^{(T)}]
    =\frac{\sigma_T^2}{n}\operatorname{tr}\!\big( X_{\mathrm{test}} (X^\top X)^{-1} X_{\mathrm{test}}^\top \big),
\]
which shows linear dependence on $\sigma_T^2$.
\end{proof}

\subsection{Proof of Theorem~\ref{thm:ground_truth_mse}}

\begin{proof}
Given Equation~\ref{eq:student model parameter}, the student prediction on the test inputs is
\[
X_{\mathrm{test}}\widehat{\theta}_S 
= X_{\mathrm{test}}\theta_T 
  + X_{\mathrm{test}}(X^\top X)^{-1}X^\top \varepsilon.
\]
Subtracting the true targets gives
\[
X_{\mathrm{test}}\widehat{\theta}_S - y^{\mathrm{true}}
= X_{\mathrm{test}}(\theta_T - \theta^*) 
+ X_{\mathrm{test}}(X^\top X)^{-1}X^\top \varepsilon 
- \eta.
\]
For clarity, define:
\[
A = X_{\mathrm{test}}(\theta_T - \theta^*), \quad
B = X_{\mathrm{test}}(X^\top X)^{-1}X^\top \varepsilon, \quad
C = \eta.
\]
Then
\[
X_{\mathrm{test}}\widehat{\theta}_S - y^{\mathrm{true}} = A + B - C
\]
and the evaluation loss is
\[
\mathcal{L}_{\mathrm{eval}} 
= \frac{1}{n}\|A + B - C\|^2
= \frac{1}{n}(A + B - C)^\top(A + B - C).
\]
Expanding the quadratic form:
\[
\mathcal{L}_{\mathrm{eval}} 
= \tfrac{1}{n}\big(A^\top A + B^\top B + C^\top C 
+ 2A^\top B - 2A^\top C - 2B^\top C\big).
\]

Then, we take $\mathbb{E}[\cdot]$ over both $\varepsilon$ and $\eta$. The term $A^\top A$ is deterministic, so
\(
\mathbb{E}[A^\top A] = A^\top A = \|X_{\mathrm{test}}(\theta_T - \theta^*)\|^2.
\)
Since $B$ depends linearly on $\varepsilon$, \[
B^\top B = \varepsilon^\top (X_{\mathrm{test}}(X^\top X)^{-1}X^\top)^\top X_{\mathrm{test}}(X^\top X)^{-1}X^\top \varepsilon.
\]

Using $\mathbb{E}[\varepsilon\varepsilon^\top] = \sigma_T^2 I_n$ and the trace identity $\mathbb{E}[\varepsilon^\top M \varepsilon] = \operatorname{tr}(M\,\mathbb{E}[\varepsilon\varepsilon^\top])$ for any matrix $M$, we obtain
\begin{equation*}
\begin{split}
\mathbb{E}[B^\top B]
&= \operatorname{tr}((X_{\mathrm{test}}(X^\top X)^{-1}X^\top)^\top X_{\mathrm{test}}(X^\top X)^{-1}X^\top\, \mathbb{E}[\varepsilon\varepsilon^\top])\\
&= \sigma_T^2 \operatorname{tr}(X^\top X(X^\top X)^{-1}X_{\mathrm{test}}^\top X_{\mathrm{test}}(X^\top X)^{-1})\\
&= \sigma_T^2 \operatorname{tr}\!\big( X_{\mathrm{test}} (X^\top X)^{-1} X_{\mathrm{test}}^\top \big)
\end{split}
\end{equation*}

For $C^\top C$, since $\mathbb{E}[\eta\eta^\top] = \sigma_\eta^2 I_n$,
\(
\mathbb{E}[C^\top C] = \operatorname{tr}(\sigma_\eta^2 I_n) = n\sigma_\eta^2.
\)

Finally, all cross terms $A^\top B$, $A^\top C$, and $B^\top C$ depends on $\mathbb{E}[\varepsilon]$ or $\mathbb{E}[\eta]$, thus $\mathbb{E}[A^\top B] = \mathbb{E}[A^\top C] = \mathbb{E}[B^\top C]= 0$ since $\varepsilon$ and $\eta$ are independent and both have zero mean.

Substituting the expectations back into the expansion,
\begin{equation*}
    \begin{split}
        \mathbb{E}[\mathcal{L}_{\mathrm{eval}}]
= \tfrac{1}{n}\|X_{\mathrm{test}}(\theta_T - \theta^*)\|
+ \tfrac{\sigma_T^2}{n}\operatorname{tr}\!\big( X_{\mathrm{test}} (X^\top X)^{-1} X_{\mathrm{test}}^\top \big)
+ \sigma_\eta^2
    \end{split}
\end{equation*}

The first term represents the bias from mismatch between the teacher and ground-truth parameters, the second term scales linearly with the teacher noise variance $\sigma_T^2$, and the third term is the irreducible evaluation noise.
\end{proof}

\subsection{Proof of Theorem~\ref{thm:teacher_bootstrap_degenerate}}
\label{app:proof_teacher_bootstrap_degenerate}

\begin{proof}
Recall the setting of Theorem~\ref{thm:teacher_bootstrap_degenerate}. The teacher is the OLS solution $\hat{\theta}_T = (X^\top X)^{-1} X^\top y$, 
and the teacher-labeled dataset is $\mathcal{D}_T = \{(x_i, \hat{y}_{T,i})\}_{i=1}^n$ and $\hat{y}_{T,i} = x_i^\top \hat{\theta}_T$. Let $I = (i_1,\dots,i_m)$ denote a bootstrap sample of indices drawn with replacement from $\{1,\dots,n\}$. The students are then trained on 
\[
    X_I \in \mathbb{R}^{m \times d}
    \quad\text{with rows } x_{i_j}^\top,
    \qquad
    \hat{y}_{T,I} \in \mathbb{R}^m
    \quad\text{with entries } \hat{y}_{T,i_j}.
\]
By construction, $\hat{y}_{T,i_j} = x_{i_j}^\top \hat{\theta}_T$, thus $\hat{y}_{T,I} = X_I \hat{\theta}_T$.

The student is defined as the OLS estimator on the bootstrapped teacher-labeled data, which gives $\hat{\theta}_S(I) = (X_I^\top X_I)^{-1} X_I^\top \hat{y}_{T,I}$. Substituting $\hat{y}_{T,I} = X_I \hat{\theta}_T$ yields
\[
    \hat{\theta}_S(I)
    = (X_I^\top X_I)^{-1} X_I^\top X_I \hat{\theta}_T
    = \hat{\theta}_T.
\]
Thus, for every bootstrap sample $I$ with $X_I^\top X_I$ invertible, the student parameter coincides exactly with the teacher parameter, and for any test input $x$ we have
\[
    f_S(x; \hat{\theta}_S(I)) = x^\top \hat{\theta}_S(I) = x^\top \hat{\theta}_T.
\]
Since this holds for all bootstrap replicates $b$, the empirical bootstrap distribution of $f_S^{(b)}(x)$ is a point mass at $x^\top \hat{\theta}_T$, and therefore the bootstrap predictive variance is zero.
\end{proof}

\subsection{Proof of Theorem~\ref{thm:groundtruth_bootstrap_variance}}

\begin{proof}
We work under the standard linear model, with $X \in \mathbb{R}^{n \times d}$ having i.i.d.\ rows and second-moment matrix $\Sigma_X = \mathbb{E}[x x^\top]$ that is positive definite. Let $I = (i_1,\dots,i_m)$ be a bootstrap sample of indices drawn with replacement from $\{1,\dots,n\}$, and define $X_I$ and $y_I$ analogously to those in Appendix~\ref{app:proof_teacher_bootstrap_degenerate}.

The student is the OLS estimator on $(X_I, y_I)$,
\[
    \hat{\theta}_S(I) = (X_I^\top X_I)^{-1} X_I^\top y_I.
\]
Using $y_I = X_I \theta + \eta_I$, where $\eta_I \in \mathbb{R}^m$ collects the resampled noise terms, we can write
\[
    \hat{\theta}_S(I)
    = (X_I^\top X_I)^{-1} X_I^\top (X_I \theta + \eta_I)
    = \theta + (X_I^\top X_I)^{-1} X_I^\top \eta_I.
\]
Conditioned on $X_I$, we have $\mathbb{E}[\eta_I \mid X_I] = 0$ and $\text{Var}(\eta_I \mid X_I) = \sigma^2 I_m$, so $\mathbb{E}[\hat{\theta}_S(I) \mid X_I] = \theta$, and $\text{Var}[\hat{\theta}_S(I) \mid X_I] = \sigma^2 (X_I^\top X_I)^{-1}$.

For sufficiently large $m$, the empirical second-moment matrix of the bootstrapped design,
\[
    \frac{1}{m} X_I^\top X_I
    = \frac{1}{m} \sum_{j=1}^m x_{i_j} x_{i_j}^\top,
\]
converges in probability to the population second-moment matrix $\Sigma_X = \mathbb{E}[x x^\top]$ by the law of large numbers. Thus, $X_I^\top X_I \approx m \Sigma_X$, and hence
\[
    (X_I^\top X_I)^{-1}
    \approx (m \Sigma_X)^{-1}
    = \frac{1}{m} \Sigma_X^{-1}.
\]

Now fix a test input $x \in \mathbb{R}^d$. The student prediction at $x$ is $f_S(x; \hat{\theta}_S(I)) = x^\top \hat{\theta}_S(I)$, and the bootstrap-induced predictive variance is
\begin{equation*}
    \begin{split}
        \text{Var}_{\mathrm{boot}}\big[f_S(x)\big]
    &= \text{Var}_{\mathrm{boot}}\big[x^\top \hat{\theta}_S(I)\big]\\
    &= x^\top \big( \sigma^2 (X_I^\top X_I)^{-1} \big) x\\
    &\approx \frac{\sigma^2}{m} \, x^\top \Sigma_X^{-1} x
    \end{split}
\end{equation*}

Thus, the ground-truth-bootstrap-induced predictive variance at $x$ decays on the order of $1/m$ as the bootstrap sample size increases, as claimed in Theorem~\ref{thm:groundtruth_bootstrap_variance}.
\end{proof}

\subsection{Proof of Theorem~\ref{thm:groundtruth_bootstrap_mse}}

\begin{proof}
As shown in the proof of Theorem~\ref{thm:groundtruth_bootstrap_variance}, $\hat{\theta}_S(I) = \theta + (X_I^\top X_I)^{-1} X_I^\top \eta_I$ and $\text{Var}[\hat{\theta}_S(I) \mid X_I] = \sigma^2 (X_I^\top X_I)^{-1}$.

Let $(x,y)$ be an independent test pair drawn from the same distribution as the training data. The student’s prediction at $x$ is $f_S(x; \hat{\theta}_S(I)) = x^\top \hat{\theta}_S(I)$. The expected test MSE is $R_m = \mathbb{E}\big[(x^\top \hat{\theta}_S(I) - y)^2\big]$.

Using $y = x^\top \theta + \eta'$ with $\eta'$ independent of $\hat{\theta}_S(I)$ and distributed as $\mathcal{N}(0,\sigma^2)$, we obtain
\[
    x^\top \hat{\theta}_S(I) - y
    =
    x^\top \hat{\theta}_S(I) - x^\top \theta - \eta'
    =
    x^\top (\hat{\theta}_S(I) - \theta) - \eta'.
\]
By independence and zero mean of $\eta'$, a bias--variance decomposition gives
\begin{align*}
    R_m
    &= \mathbb{E}\big[(x^\top (\hat{\theta}_S(I) - \theta) - \eta')^2\big] \\
    &= \mathbb{E}\big[(x^\top (\hat{\theta}_S(I) - \theta))^2\big]
       + \mathbb{E}[(\eta')^2]
       - 2\,\mathbb{E}\big[x^\top (\hat{\theta}_S(I) - \theta)\,\eta'\big] \\
    &= \mathbb{E}\big[(x^\top (\hat{\theta}_S(I) - \theta))^2\big]
       + \sigma^2,
\end{align*}
since the cross term vanishes by independence and zero mean, and $\mathbb{E}[(\eta')^2] = \sigma^2$.

Conditioning on $X_I$ and using $\mathbb{E}[\hat{\theta}_S(I) \mid X_I] = \theta$, we have
\begin{equation*}
    \begin{split}
        \mathbb{E}\big[(x^\top (\hat{\theta}_S(I) - \theta))^2\big]
    &= \mathbb{E}\big[x^\top \text{Var}(\hat{\theta}_S(I) \mid X_I)\, x\big]\\
    &= \sigma^2\, \mathbb{E}\big[x^\top (X_I^\top X_I)^{-1} x\big].
    \end{split}
\end{equation*}
From the previous proof, $(X_I^\top X_I)^{-1} \approx \frac{1}{m} \Sigma_X^{-1}$. Substituting this approximation yields
\begin{equation*}
    \mathbb{E}\big[x^\top (X_I^\top X_I)^{-1} x\big] \approx \frac{1}{m} \,\mathbb{E}\big[x^\top \Sigma_X^{-1} x\big]\approx \frac{1}{m} \,\mathbb{E}\big[\mathrm{tr}(x^\top \Sigma_X^{-1} x)\big].
\end{equation*}
By properties of trace, 
\begin{equation*}
    \mathbb{E}\big[x^\top (X_I^\top X_I)^{-1} x\big] 
    \approx \frac{1}{m} \,\mathrm{tr}\big(\Sigma_X^{-1} \mathbb{E}[x x^\top]\big) \approx \frac{1}{m} \,\mathrm{tr}(\Sigma_X^{-1} \Sigma_X)\approx \frac{1}{m} \,d.
\end{equation*}
Combining the above expressions, we obtain
\[
    \mathbb{E}\big[(x^\top (\hat{\theta}_S(I) - \theta))^2\big]
    \approx \frac{\sigma^2}{m} \, d,
\]
and hence
\[
    R_m
    = \sigma^2 + \mathbb{E}\big[(x^\top (\hat{\theta}_S(I) - \theta))^2\big]
    \approx \sigma^2 + \frac{\sigma^2 d}{m}.
\]

The right-hand side is a decreasing function of $m$ that converges to $\sigma^2$ as $m \to \infty$. If we consider a no-bootstrap student trained once on $n$ samples by OLS, the same calculation (with $m$ replaced by $n$) yields an approximate expected test MSE
\(
    R_n \approx \sigma^2 + \frac{\sigma^2 d}{n},
\)
so $R_m$ approaches $R_n$ as $m$ increases towards $n$. This completes the proof.
\end{proof}

\subsection{Proof of Theorem~\ref{thm:averaging_var}}

\begin{proof}
For each training input \(x_i\), since the teacher responses are i.i.d.\ with mean \(f_T(x_i)\) and variance \(\sigma_T^2\), the sample mean satisfies
\[
\mathbb{E}[\mu_{T,i}] = f_T(x_i), 
\qquad
\operatorname{Var}[\mu_{T,i}] = \frac{\sigma_T^2}{k}.
\]
Stacking these over all inputs gives
\(\operatorname{Var}[\mu_T] = \frac{\sigma_T^2}{k} I_n.\)
Substituting $\mu_T$ in place of $y^{(T)}$ in the OLS solution
\(\widehat{\theta}_S = (X^\top X)^{-1} X^\top \mu_T\),
we have
\begin{equation*}
    \begin{split}
        \operatorname{Var}(\widehat{\theta}_S) = (X^\top X)^{-1} X^\top \operatorname{Var}(\mu_T) X (X^\top X)^{-1}
        = \frac{\sigma_T^2}{k}(X^\top X)^{-1}.
    \end{split}
\end{equation*}

Hence, both the dataset noise and the resulting student parameter variance scale inversely with \(k\) and decay at a rate of \(\mathcal{O}(1/k)\).

For a fixed test input \(x_\star\) the student prediction is \(\widehat{f}_S(x_\star)=x_\star^\top\widehat{\theta}_S\). Hence
\[
\operatorname{Var}\big(\widehat{f}_S(x_\star)\big)
= x_\star^\top \operatorname{Var}(\widehat{\theta}_S) x_\star
= \frac{\sigma_T^2}{k}\; x_\star^\top (X^\top X)^{-1} x_\star,
\]
which exhibits the claimed \(\mathcal{O}(1/k)\) dependence.
\end{proof}

\subsection{Proof of Theorem~\ref{thm:averaging_umvue}}

\begin{proof}
Each teacher response \(y_{i,j}^{(T)}\) is normally distributed with mean \(f_T(x_i)\) and variance \(\sigma_T^2\). The sample mean \(\mu_{T,i}\) is an unbiased estimator of \(f_T(x_i)\) and is a function of the complete sufficient statistic for the normal family (the sum of observations).  

By the Lehmann-Scheff\'e theorem~\cite{lehmann2011completeness1,lehmann2011completeness2}, any unbiased estimator that is a function of a complete sufficient statistic is the unique uniformly minimum-variance unbiased estimator (UMVUE). Therefore, \(\mu_{T,i}\) is the UMVUE of \(f_T(x_i)\).
\end{proof}

\subsection{Proof of Theorem~\ref{thm:var_weight_optimal}}

\begin{proof}
First, 
\begin{equation*}
    \begin{split}
        \operatorname{Var}[\widehat{y}(x_i)] &= \operatorname{Var}[w_{T,i}\,\mu_{T,i} + w_{S,i}\,\mu_{S,i}]\\
        &= w_{T,i}^2\cdot \operatorname{Var}[\mu_{T,i}] + w_{S,i}^2\cdot \operatorname{Var}[\mu_{S,i}]\\
        &= w_{T,i}^2\cdot \frac{\sigma^2_{T,i}}{k} + w_{S,i}^2\cdot \frac{\sigma^2_{S,i}}{k}\\
    \end{split}
\end{equation*}

Substituting \(w_{S,i} = 1 - w_{T,i}\) into the expression for the variance yields
\[
\operatorname{Var}[\widehat{y}(x_i)] 
= \frac{1}{k}\;w_{T,i}^2\widehat{\sigma_{T,i}^2} + \frac{1}{k}\;(1 - w_{T,i})^2\sigma_{S,i}^2.
\]
Differentiating with respect to \(w_{T,i}\) and setting the derivative to zero gives:
\begin{equation*}
    \begin{split}
        \frac{\partial \operatorname{Var}[\widehat{y}(x_i)]}{\partial w_{T,i}}
&= \frac{2}{k}\;w_{T,i}\widehat{\sigma_{T,i}^2} - \frac{2}{k}\;(1 -w_{T,i})\sigma_{S,i}^2 = 0\\
w_{T,i}\widehat{\sigma_{T,i}^2} &= \sigma_{S,i}^2 - w_{T,i}\sigma_{S,i}^2\\
w_{T,i} &= \frac{\sigma_{S,i}^2}{\widehat{\sigma_{T,i}^2} + \sigma_{S,i}^2}= \frac{1/\widehat{\sigma_{T,i}^2}}{1/\widehat{\sigma_{T,i}^2} + 1/\sigma_{S,i}^2}
    \end{split}
\end{equation*}

It follows that
\[
w_{S,i} = 1 - w_{T,i} = \frac{1/\sigma_{S,i}^2}{1/\widehat{\sigma_{T,i}^2} + 1/\sigma_{S,i}^2}.
\]

The second derivative is
\[
\frac{\partial^2 \operatorname{Var}[\widehat{y}(x_i)]}{\partial w_{T,i}^2} = \frac{2}{k}\;\widehat{\sigma_{T,i}^2} + \frac{2}{k}\;\sigma_{S,i}^2>0,
\]
so this solution is the unique global minimizer.
\end{proof}

\subsection{Proof of Theorem~\ref{thm:var_weight_properties}}

\begin{proof}
From the previous result,
\[
\operatorname{Var}[\widehat y]_{\min}
= \frac{1}{k}\, w_{T,i}^2 \widehat{\sigma}_{T,i}^2
  + \frac{1}{k}\, w_{S,i}^2 \sigma_{S,i}^2 .
\]
Substituting the optimal inverse-variance weights
\[
w_{T,i} = \frac{1/\widehat{\sigma}_{T,i}^2}{1/\widehat{\sigma}_{T,i}^2 + 1/\sigma_{S,i}^2},
\qquad
w_{S,i} = \frac{1/\sigma_{S,i}^2}{1/\widehat{\sigma}_{T,i}^2 + 1/\sigma_{S,i}^2},
\]
and simplifying yields
\[
\operatorname{Var}[\widehat y]_{\min}
= \frac{1}{k}\cdot
\frac{\widehat{\sigma}_{T,i}^2\,\sigma_{S,i}^2}
{\widehat{\sigma}_{T,i}^2 + \sigma_{S,i}^2}.
\]


To prove property (1), note that for any positive $a,b$,
\[
\frac{ab}{a+b} < a.
\]
Applying this with $a=\widehat{\sigma}_{T,i}^2$ and $b=\sigma_{S,i}^2$ yields
\[
\operatorname{Var}[\widehat y]_{\min}
= \frac{1}{k}\cdot \frac{\widehat{\sigma}_{T,i}^2\sigma_{S,i}^2}{\widehat{\sigma}_{T,i}^2+\sigma_{S,i}^2}
< \frac{1}{k}\cdot \widehat{\sigma}_{T,i}^2,
\]
as claimed.
The proof showing that \(\operatorname{Var}[\widehat{y}]_{\min} < \frac{1}{k}\cdot \sigma_{S,i}^2\) follows similarly.

The rest of the properties follow trivially from the variance formulation.
\end{proof}

\section{Additional Experimental Details}
\label{app:additional-details}

\subsection{Dataset Descriptions}
\label{app:datasets_all}

We provide the detailed descriptions of the datasets used in our study here. All input features are standardized to zero mean and unit variance, and the data are split into training and test sets with an 80/20 ratio, unless otherwise noted. Classification datasets use a stratified split to preserve class proportions.
\begin{itemize}
    \item \textbf{Boston Housing dataset~\cite{harrison1978hedonic}} is a standard regression benchmark comprising 506 samples of median house prices in the Boston metropolitan area. Each sample is characterized by 13 numerical and categorical attributes, such as the average number of rooms, local crime rate, and property tax rate.
    \item \textbf{BioASQ~\cite{krithara2023bioasq}} is a biomedical QA benchmark consisting of expert-annotated training and test pairs. The teacher is fine-tuned on the BioASQ training set, after which knowledge distillation is applied to 10 pre-trained DistilGPT2 models during supervised fine-tuning to approximate the benefits of knowledge distillation while maintaining computational efficiency~\cite{nguyen2025smoothing}.
    \item \textbf{Digits dataset~\cite{optical_recognition_of_handwritten_digits_80}} consists of $8\times8$ grayscale images of handwritten digits (0--9), flattened into 64-dimensional feature vectors. 
\end{itemize}

\begin{table*}[t]
\centering\small
\caption{Empirical validation of the entropy inequality across teacher model temperatures.}
\label{tab:entropy_comparison}
\begin{tabular}{c|cccccc}
\toprule
Teacher Model Temp. & 0.5 & 0.8 & 1.0 & 1.2 & 1.5 & 2.0 \\
\midrule
\% Prompts with $H(\mathbb{E}[p_S])\le H(p_T)$ 
& 99.96\% & 99.98\% & 100\% & 100\% & 100\% & 100\% \\
\bottomrule
\end{tabular}
\end{table*}

\subsection{Predictive Distribution Experiment Setup in LLM Distillation}
\label{app:llm_predictive_distribution}

We estimate predictive distributions for both teacher and student LLM models using 1,000 generated responses per model, subsequently projecting them onto a single principal component for clarity. For visualization, only one student model is shown, though the pattern is consistent across multiple students.

\subsection{Bootstrap Experiment Setup}
\label{app:bootstrap_setup}

We validate Theorems~\ref{thm:teacher_bootstrap_degenerate}, \ref{thm:groundtruth_bootstrap_variance}, and \ref{thm:groundtruth_bootstrap_mse} on the Boston Housing regression task using linear regression and a one-hidden-layer multilayer perceptron.

For bootstrapping, we consider a grid of bootstrap sample sizes $m \in \{ \beta n : \beta \in [0.1, 1.0] \}$ (rounded to the nearest integer). For each $m$ and each variant (teacher bootstrap and ground-truth bootstrap), we draw $B$ bootstrap samples (we use $B=1,000$ in our experiments) and fit one student per sample. For every student, we compute the test MSE on the fixed test set. For each value of $m$, we then report the mean and variance of the test MSE across the $B$ students for both variants.

\section{Suppression Effect in Model Predictive Distribution}
\label{app:suppression_effect}

We formalize the suppression effect in the student model predictive distribution discussed in Section~\ref{sec:uncertainty_llm} using an entropy inequality.

\begin{theorem}[Suppression effect]
\label{thm:lower student var}
Suppose the student model is trained using one sampled teacher response per input. Assuming that teacher outputs follow $p_T(\mathbf{y}\mid\mathbf{x})$, if the entropy $H(\mathbb{E}[p_S(\mathbf{y}\mid\mathbf{x})]) \le H(p_T(\mathbf{y}\mid\mathbf{x}))$, then $\mathbb{E}[H(p_S(\mathbf{y}\mid\mathbf{x}))] \le H(p_T(\mathbf{y}\mid\mathbf{x}))$.
\qed
\end{theorem}

\begin{proof}
Entropy is concave. Applying Jensen’s inequality gives
\(
\mathbb{E}[H(p_S(\mathbf{y}\mid\mathbf{x}))] \le H(\mathbb{E}[p_S](\mathbf{y}\mid\mathbf{x})) \le H(p_T(\mathbf{y}\mid\mathbf{x})).
\)
\end{proof}

This shows that, under standard distillation, student models are biased toward lower-entropy (narrower) distributions than their teacher model. 



To test the needed assumption, we approximate entropy via variance in SBERT embedding space using multiple samples from teacher and student models. Across a wide range of teacher model temperatures, we find that the assumption holds in nearly all cases (Table~\ref{tab:entropy_comparison}), validating the theorem in realistic LLM settings.

Note that this theorem does not contradict the results in Section~\ref{sec:llm_student_uncertainty_intra_model}. In that experiment, the teacher model’s outputs used for distillation are fixed, so the teacher model’s predictive entropy during training is effectively zero. By contrast, the student models still produce stochastic outputs at inference time, so Theorem~\ref{thm:lower student var}'s assumptions do not apply.

\section{Optimization Trajectory Analysis for LLM Student Initialization Uncertainty}
\label{app:param_trajectories}

\begin{table*}[t]
\centering
\small
\caption{Predictive entropy of teacher and student models for logistic regression. Values in parentheses in the Dataset column denote accuracy. Significance is determined by overlapping intervals $H_T \pm \sigma_T$ and $H_S \pm \sigma_S$.}
\label{tab:kd_entropy_logreg}
\begin{tabular}{llcccccc}
\toprule
Dataset & Group & $N_T$ & $N_S$ & $H_T \pm \sigma_T$ & $H_S \pm \sigma_S$ & Significant? \\
\midrule
\multirow{3}{*}{Wine (T=97.22\%, S=97.22\%)} & All & 36 & 36 & 0.125 $\pm$ 0.198 & 0.132 $\pm$ 0.210 & No \\
 & Correct & 35 & 35 & 0.109 $\pm$ 0.176 & 0.116 $\pm$ 0.189 & No \\
 & Incorrect & 1 & 1 & 0.689 $\pm$ 0.000 & 0.693 $\pm$ 0.000 &  ---\\
\midrule
\multirow{3}{*}{Breast Cancer (T=97.37\%, S=97.37\%)} & All & 114 & 114 & 0.109 $\pm$ 0.192 & 0.070 $\pm$ 0.160 & No \\
 & Correct & 111 & 111 & 0.097 $\pm$ 0.179 & 0.060 $\pm$ 0.146 & No \\
 & Incorrect & 3 & 3 & 0.539 $\pm$ 0.164 & 0.444 $\pm$ 0.200 & No \\
\midrule
\multirow{3}{*}{Digits (T=96.39\%, S=96.39\%)} & All & 360 & 360 & 0.188 $\pm$ 0.307 & 0.179 $\pm$ 0.308 & No \\
 & Correct & 347 & 347 & 0.163 $\pm$ 0.278 & 0.152 $\pm$ 0.275 & No \\
 & Incorrect & 13 & 13 & 0.847 $\pm$ 0.325 & 0.896 $\pm$ 0.291 & No \\
\midrule
\multirow{3}{*}{MNIST (T=91.90\%, S=91.85\%)} & All & 14000 & 14000 & 0.233 $\pm$ 0.320 & 0.094 $\pm$ 0.213 & No \\
 & Correct & 12866 & 12859 & 0.189 $\pm$ 0.270 & 0.062 $\pm$ 0.162 & No \\
 & Incorrect & 1134 & 1141 & 0.739 $\pm$ 0.396 & 0.449 $\pm$ 0.352 & No \\
\midrule
\multirow{3}{*}{Covertype (T=71.75\%, S=71.69\%)} & All & 10000 & 10000 & 0.800 $\pm$ 0.208 & 0.408 $\pm$ 0.277 & No \\
 & Correct & 7175 & 7169 & 0.769 $\pm$ 0.208 & 0.376 $\pm$ 0.267 & No \\
 & Incorrect & 2825 & 2831 & 0.881 $\pm$ 0.185 & 0.489 $\pm$ 0.285 & No \\
\bottomrule
\end{tabular}
\end{table*}

This appendix provides additional analysis supporting the initialization sensitivity results reported in Section~\ref{sec:init-model-driven} for LLM student models. To further investigate how the optimization trajectories behave under perturbation, we track the optimization trajectory of one parameter of DistilGPT2 under varying levels of perturbation. The results are visualized in Figure~\ref{fig:param_trajectories}.

At low noise levels (1–5\%), the perturbed trajectories initially resemble the path taken by the unperturbed parameter, yet they fail to converge to the same optimum. Instead, each trajectory settles into a nearby but distinct point in parameter space. As the magnitude of noise increases, this divergence becomes more pronounced. The optimization paths deviate earlier and further, ultimately converging to substantially different optima. 

These results indicate that the knowledge distillation signal is relatively weak as a corrective force. While it provides a small force that pushes the student toward the teacher's outputs, the signal is not strong enough to dominate the influence of parameter initialization or to guarantee recovery of the teacher-aligned optimum, and small perturbations in initialization can determine which local basin the optimizer ultimately enters. As a result, initialization-driven uncertainty can persist even when supervision and optimization settings are otherwise identical, explaining the pronounced inter-student variability observed in LLM distillation in Section~\ref{sec:init-model-driven}.

\begin{figure}[t]
\centering
\includegraphics[width=0.8\linewidth]{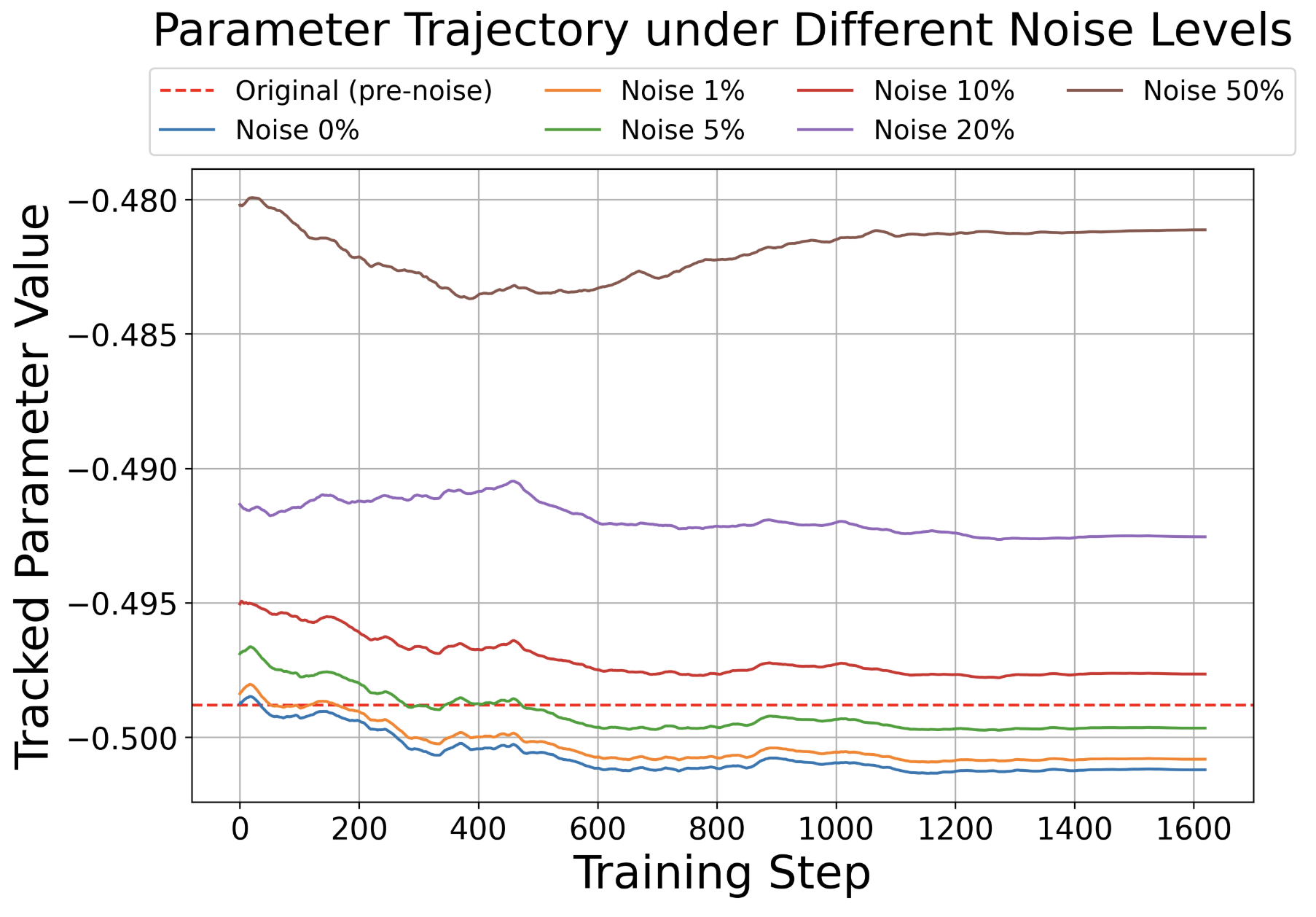}
    \caption{Parameter trajectories under different initialization noise levels during LLM distillation. Perturbed runs diverge to distinct optima.}
    \label{fig:param_trajectories}
\end{figure}

\section{Additional Experiment Setup and Results for Student Output Uncertainty}
\label{app:student_output_uncertainty_extra_datasets}

This appendix summarizes the additional experimental details and results for assessing student output uncertainty in classification tasks beyond the Digits dataset described in Section~\ref{sec:simple_student_output_uncertainty}.

\subsection{Model Details}

We tested both logistic regression and neural networks. The neural network teacher model maps the 64-dimensional inputs to a hidden layer of 128 units and then to a 10-class output; the student model has the same structure but with 64 hidden units. Both are trained using cross-entropy loss and the Adam optimizer for 300 epochs.

\subsection{Additional Datasets}
\label{app:dataset_descriptions}

All datasets are treated as supervised classification problems. Unless otherwise noted, we perform an 80/20 stratified split into training and test sets and standardize all features to zero mean and unit variance using a \texttt{StandardScaler} fitted on the training data. The description for the Digits dataset can be found in Appendix~\ref{app:datasets_all}.

\paragraph{Wine.}
The \emph{Wine} dataset~\cite{uci-wine} contains 178 examples of three different types of wine. Each example is described by 13 continuous attributes such as alcohol, ash, and flavonoid content. We treat this as a 13-dimensional, 3-class classification task.

\paragraph{Breast Cancer.}
The \emph{Breast Cancer Wisconsin (Diagnostic)} dataset~\cite{uci-breast-cancer} consists of 569 instances of digitized fine-needle aspirates of breast masses. Each instance is represented by 30 real-valued features describing characteristics of cell nuclei (e.g., radius, texture, smoothness). The labels indicate whether the tumor is benign or malignant, yielding a binary classification problem.

\paragraph{MNIST}
The \emph{MNIST} dataset~\cite{lecun-mnist} comprises 70,000 grayscale images of handwritten digits (0--9), split into 60,000 training and 10,000 test images in the canonical configuration. Each image is $28\times28$ pixels. 
In our implementation, we load the standard \texttt{torchvision} MNIST training and test sets, concatenate them, and then perform our own 80/20 stratified split. 
Images are flattened into 784-dimensional vectors and standardized. This yields a 10-class problem of greater difficulty compared to the lower-resolution Digits dataset.

\paragraph{Covertype.}
The \emph{Covertype} dataset~\cite{blackard1998covertype} involves predicting forest cover type from cartographic variables. The full dataset contains 581,012 instances described by 54 numerical and binary features (elevation, slope, wilderness area indicators, soil types, etc.) and labeled into 7 cover types. To keep computation manageable while preserving difficulty, we draw a random subsample of up to 50,000 examples. This yields a 54-dimensional, 7-class classification task.

\begin{table*}[t]
\centering
\small
\caption{Predictive entropy of teacher and student models for neural networks. Values in parentheses in the Dataset column denote accuracy. Significance is determined by overlapping intervals $H_T \pm \sigma_T$ and $H_S \pm \sigma_S$.}
\label{tab:kd_entropy_nn}
\begin{tabular}{llcccccc}
\toprule
Dataset & Group & $N_T$ & $N_S$ & $H_T \pm \sigma_T$ & $H_S \pm \sigma_S$ & Significant? \\
\midrule
\multirow{3}{*}{Wine (T=97.22\%, S=97.22\%)} & All & 36 & 36 & 0.024 $\pm$ 0.070 & 0.043 $\pm$ 0.121 & No \\
 & Correct & 35 & 35 & 0.018 $\pm$ 0.062 & 0.036 $\pm$ 0.113 & No \\
 & Incorrect & 1 & 1 & 0.230 $\pm$ 0.000 & 0.317 $\pm$ 0.000 & --- \\
\midrule
\multirow{3}{*}{Breast Cancer (T=94.74\%, S=94.74\%)} & All & 114 & 114 & 0.019 $\pm$ 0.090 & 0.019 $\pm$ 0.086 & No \\
 & Correct & 108 & 108 & 0.009 $\pm$ 0.064 & 0.006 $\pm$ 0.037 & No \\
 & Incorrect & 6 & 6 & 0.212 $\pm$ 0.204 & 0.256 $\pm$ 0.235 & No \\
\midrule
\multirow{3}{*}{Digits (T=97.22\%, S=97.22\%)} & All & 360 & 360 & 0.035 $\pm$ 0.127 & 0.044 $\pm$ 0.156 & No \\
 & Correct & 350 & 350 & 0.023 $\pm$ 0.099 & 0.035 $\pm$ 0.142 & No \\
 & Incorrect & 10 & 10 & 0.446 $\pm$ 0.252 & 0.387 $\pm$ 0.240 & No \\
\midrule
\multirow{3}{*}{MNIST (T=97.04\%, S=96.54\%)} & All & 14000 & 14000 & 0.027 $\pm$ 0.118 & 0.035 $\pm$ 0.133 & No \\
 & Correct & 13585 & 13515 & 0.018 $\pm$ 0.090 & 0.022 $\pm$ 0.100 & No \\
 & Incorrect & 415 & 485 & 0.348 $\pm$ 0.316 & 0.380 $\pm$ 0.327 & No \\
\midrule
\multirow{3}{*}{Covertype (T=80.96\%, S=79.29\%)} & All & 10000 & 10000 & 0.459 $\pm$ 0.270 & 0.221 $\pm$ 0.253 & No \\
 & Correct & 8096 & 7929 & 0.411 $\pm$ 0.257 & 0.183 $\pm$ 0.231 & No \\
 & Incorrect & 1904 & 2071 & 0.663 $\pm$ 0.221 & 0.368 $\pm$ 0.278 & No \\
\bottomrule
\end{tabular}
\end{table*}

\subsection{Experimental Results: Logistic Regression}
\label{app:results_logreg}

For each dataset, we train a logistic regression teacher on the ground-truth labels and then train a logistic regression student solely on hard pseudo-labels produced by the teacher. As described in Section~\ref{sec:simple_student_output_uncertainty}, we quantify deterministic student output uncertainty using predictive entropy on the test set. 

The full results for logistic regression are shown in Table~\ref{tab:kd_entropy_logreg}. Columns $N_T$ and $N_S$ denote the number of test examples in each group for the teacher and student, respectively. The ``Significant?'' column indicates whether the intervals $H_T \pm \sigma_T$ and $H_S \pm \sigma_S$ are disjoint; overlapping intervals are treated as evidence of no statistically significant differences. 

Across all datasets, the distilled student closely matches the teacher's test accuracy and exhibits very similar uncertainty profiles. Even in cases where the mean entropies differ slightly (e.g., MNIST and Covertype), the standard deviation is large, leading to overlapping intervals in all cases, and we therefore do not regard these differences as significant.

\subsection{Experimental Results: Neural Networks}
\label{app:results_nn}

We repeat the same experiment using one-hidden-layer neural networks as teacher and student models. Table~\ref{tab:kd_entropy_nn} reports the predictive entropy statistics for neural networks. Teacher and student accuracies are again very close on all datasets. The mean entropies differ slightly in some datasets (e.g., MNIST and Covertype), but in all cases the teacher and student intervals $H_T \pm \sigma_T$ and $H_S \pm \sigma_S$ overlap. We therefore do not find statistically compelling evidence of a systematic change in the model output uncertainty.

\begin{table}[t]
\centering
\scriptsize
\caption{Noise characteristics of LLM distillation with a single teacher response under varying teacher temperatures. Both the slope and $R^2$ decrease monotonically with increasing teacher temperature, indicating that as the teacher becomes noisier, students trained on a single teacher response struggle to reproduce the teacher's stochastic behavior.
}
\label{tab:noise_temp}
\begin{tabular}{ccccc}
\toprule
\textbf{Distillation Setup} & \textbf{Slope} & $\mathbf{R^2}$ & \textbf{Avg. Noise} & \textbf{Avg. Systematic Noise} \\
\midrule
Direct FT & 0.374 & 0.151 & 0.505 & -- \\
\midrule
Teacher Temp = 0.5 & 0.204 & 0.095 & 0.666 & 0.161 \\
Teacher Temp = 0.8 & 0.191 & 0.085 & 0.671 & 0.166 \\
Teacher Temp = 1.0 & 0.182 & 0.079 & 0.673 & 0.168 \\
Teacher Temp = 1.2 & 0.169 & 0.071 & 0.678 & 0.173 \\
Teacher Temp = 1.5 & 0.162 & 0.067 & 0.684 & 0.179 \\
Teacher Temp = 2.0 & 0.152 & 0.056 & 0.686 & 0.181 \\
\bottomrule
\end{tabular}
\end{table}

\section{Noise Transfer Analysis in LLM Distillation}
\label{app:noise_exps}

Apart from added variance and hallucination, one other potential side effect of the distillation pipeline is \emph{unwanted noise}, which is another aspect of uncertainty. To this end, we conduct experiments that measure both the overall and systematic noise transfer during distillation. Specifically, we address two questions:  
(1) How well do student models inherit the \emph{overall} noise of their teachers?  
(2) How well do they inherit the \emph{systematic} noise---noise induced by the teacher itself rather than by the data?

\subsection{Noise Decomposition and Measurement}
We decompose model noise into three components.  
\textbf{Non-systematic noise} is a source of irreducible, aleatoric uncertainty~\cite{kendall2017uncertainties} that arises from the intrinsic stochasticity of the data and the optimization process. It can be estimated by directly fine-tuning the student model on the ground-truth training data, without teacher supervision.  
\textbf{Overall noise} captures the total variability in student outputs when the student is distilled from a teacher that produces stochastic responses, which we control by varying the teacher's temperature when generating the distillation dataset.  
The \textbf{systematic noise}, a source of reducible, epistemic uncertainty~\cite{kendall2017uncertainties}, is the portion of variability introduced by the teacher and thus potentially reducible. It is then computed as
\[
\text{Systematic noise} = \text{Overall noise} - \text{Non-systematic noise}.
\]

During evaluation, for each input prompt, we generate 10 responses from both teacher and student models at temperature $0.8$ and measure the mean pairwise cosine distance between response embeddings. This quantity serves as a per-prompt measure of output variability. We then plot the teacher's noise against the student's noise across all prompts and compute the slope of the best-fit line and coefficient of determination $R^2$, which jointly quantify how closely the student reproduces the teacher's stochastic behavior.

\subsection{Failure Modes of Single-Response Distillation}

Table~\ref{tab:noise_temp} summarizes the results under single teacher output supervision. 
Avg. Noise refers to each student's average noise per prompt, while Avg. Systematic Noise is computed as the difference between a student's average noise and that of the direct fine-tune (FT) baseline.

The direct FT model establishes the non-systematic noise level and serves as a reference for the ideal amount of variability a well-calibrated student should exhibit after distillation. However, as teacher noise increases (via higher teacher temperature), all metrics—slope, $R^2$, and average noise—indicate growing divergence between teacher and student noise patterns. The average systematic noise also increases monotonically with teacher temperature. In particular, the average systematic noise increases from 0.161 at temperature 0.5 to 0.181 at temperature 2.0, corresponding to a 12.4\% increase. 

This trend is intuitive. As the teacher's predictive distribution becomes noisier and more stochastic, the single-response knowledge distillation setup exposes the student to only one random sample per input.  
Consequently, the student cannot capture the full variability of the teacher's distribution and instead internalizes an inconsistent and noisier approximation. This mismatch manifests as increasing systematic noise and weaker noise correlation as shown through lower $R^2$, showing that single-response knowledge distillation struggles to faithfully transmit uncertainty when teacher outputs are highly stochastic. This motivates the application of averaging and variance-weighting methods in Section~\ref{sec:reduce_noise_llm}.

\end{sloppy}

\end{document}